\pdfoutput=1
\documentclass[twoside]{article}

\usepackage[preprint]{aistats2027}

\usepackage[numbers,sort&compress]{natbib}

\usepackage[utf8]{inputenc}
\usepackage[T1]{fontenc}
\usepackage{url}
\usepackage{microtype}
\usepackage{graphicx}
\usepackage{subcaption}
\usepackage{booktabs}
\usepackage{multirow}
\usepackage[colorlinks=false]{hyperref}
\usepackage{xurl}

\usepackage{algorithm}
\usepackage{algorithmic}
\usepackage{amsmath}
\usepackage{enumitem}
\usepackage{siunitx}
\usepackage{amssymb}
\usepackage{mathtools}
\usepackage{pifont}
\usepackage{amsthm}
\usepackage[capitalize,noabbrev]{cleveref}

\usepackage{listings}
\usepackage{tikz}
\usetikzlibrary{shapes, arrows.meta, positioning, calc}
\theoremstyle{plain}
\newtheorem{theorem}{Theorem}[section]
\newtheorem{proposition}[theorem]{Proposition}
\newtheorem{lemma}[theorem]{Lemma}
\newtheorem{corollary}[theorem]{Corollary}
\theoremstyle{definition}
\newtheorem{definition}[theorem]{Definition}
\newtheorem{assumption}[theorem]{Assumption}

\theoremstyle{remark}
\newtheorem{remark}[theorem]{Remark}

\usepackage{tcolorbox}
\usepackage{placeins}
\tcbuselibrary{skins}

\newcommand{\mypara}[1]{\noindent\textbf{#1}}

\newif\ifpreprint
\preprintfalse

\newcommand{\methodname}{CIE}

\newcommand{\service}[1]{\textcolor{brown!50!black}{#1}}

\newcommand{\internal}[1]{\textcolor{blue}{#1}}
\newcommand{\external}[1]{\textcolor{green!50!black}{#1}}

\tcbset{colormetabox/.style={
    enhanced,
    boxrule=0pt,
    colback=black!12,
    borderline west={2.5pt}{0pt}{black!60},
    left=6pt, right=4pt, top=2pt, bottom=2pt,
    before skip=4pt, after skip=4pt,
}}

\newcommand{\myauthornote}[3]{}

 \preprinttrue

\begin{document}

\runningtitle{Graph-Free Root Cause Analysis}
\runningauthor{Luan Pham}

\twocolumn[
\aistatstitle{Graph-Free Root Cause Analysis}
\aistatsauthor{Luan Pham}
\aistatsaddress{RMIT University, Australia\\ \texttt{luan.pham@rmit.edu.au}}
]

\begin{abstract}
Failures in complex systems demand rapid Root Cause Analysis (RCA) to prevent cascading damage. Graph-free RCA methods typically assume that the root cause has the \textit{highest anomaly score}. This assumption fails when faults propagate, because a small deviation at the root cause can grow into a larger anomaly downstream. Our key observation is that a root cause deviates in both its internal properties (e.g., resource usage) and external properties (e.g., latency), whereas components affected by propagation deviate only in external properties. We propose \textbf{\methodname}, a framework that ranks components by combining their internal and external anomaly scores. We formalize a class of component-based systems and state conditions under which \methodname{} provably ranks the root cause first. The guarantee holds under arbitrary amplification for a conjunctive score, and under bounded amplification and a stronger root-cause condition for the default additive score. On 735 failures in nine datasets from three microservice systems, \methodname{} achieves 68\% \mbox{Top-1} accuracy averaged over 11 fault types, a 21\% relative improvement over the best of twelve baselines, while requiring 8\,ms per diagnosis.
 \end{abstract}

\section{Introduction}\label{sec:intro}

Modern software systems drive critical infrastructure, from hospital services to financial markets. When these systems fail, rapid root cause identification determines whether operators can intervene before damage cascades. Delays may lead to significant economic losses~\cite{yahoo_amazon_downtime_2018}, or even fatalities~\cite{Gregory2025Optus}. Root Cause Analysis (RCA) has thus attracted substantial research attention across domains~\cite{orchardroot, li2025root, jhaitbench, pham2025rcaeval, zhang2025adaptive}. Given a failure, RCA aims to identify the component whose mechanism changed and subsequently propagated its effects to other components, ultimately leading to the observed failure.

Graph-based RCA approaches fall into two groups. Methods that assume known structural dependencies reduce RCA to a traversal problem, where the root cause is a component with no anomalous parents~\cite{budhathoki2022causal, li2022causal}. However, this knowledge is often unavailable in large, evolving systems. Methods that construct causal graphs from observational data using causal discovery~\cite{xin2023causalrca, ikram2022root} have also shown limited success. \citet{pham2024root} evaluate twenty-one causal inference-based RCA methods and conclude that they fall short, largely due to limitations of causal discovery. The field therefore needs efficient RCA techniques that operate solely on observational data without structural dependency knowledge~\cite{orchardroot, pham2024baro}. We ask \textbf{how to determine the root cause when the dependency graph is unknown}.

A natural alternative is to rank candidates by their anomaly score, assuming that \textit{the root cause exhibits the highest anomaly score}~\cite{orchardroot,pham2024baro,shan2019diagnosis}. For instance, \citet{orchardroot} claim that ``an anomaly $x_i$ is unlikely to cause a much larger anomaly $x_j$ anywhere downstream.'' However, this assumption fails when faults propagate and amplify. Consider a service that calls a backend $k$ times sequentially per request. If the backend has a small delay $\Delta$, the service observes a cumulative delay of $k\Delta$. With independent per-call fluctuations, the caller's anomaly remains $\sqrt{k}$ times larger than the backend's even after each service's latency is normalized by its baseline variability (Theorem~\ref{thm:counter}). For a delay that is small in absolute terms but large relative to baseline noise, the caller's latency appears more anomalous than the root cause's, which inverts the expected ranking and misleads score-based methods.

Our key observation is that root causes and affected components differ in which properties deviate. Root causes show anomalies in both internal properties (e.g., config states, resource usage) and external properties (e.g., response latency, error rate). When components interact only through their interfaces and external properties do not feed back into a component's internal state, affected components show anomalies only in external properties, because faults propagate through these interfaces. In this idealized setting, downstream amplification inflates external anomalies but leaves this asymmetry intact. We formalize this insight through the Component-Property Model, which captures systems where internal properties causally precede external properties and inter-component influence occurs only through external channels. Our contributions are:
\begin{itemize}[leftmargin=*,nolistsep]
\item We propose \methodname{}, a simple and efficient RCA framework that localizes root causes from pre- and post-failure observational data without a dependency graph. It needs only each metric's component and a one-time internal/external label per metric type.
\item We introduce the Component-Property Model, which formalizes systems with internal--external property decomposition. We prove that the conjunctive \methodname{} score ranks the root cause first under arbitrary external amplification, provided its internal and external scores both exceed a threshold that bounds the internal scores of all other components. \item We evaluate \methodname{} on nine datasets from RCAEval, comprising 735 failures of 11 fault types in three microservice systems. \methodname{} achieves 68\% Top-1 accuracy, a 21\% relative improvement over the best of 12 baselines, while requiring only 8\,ms per diagnosis.
\end{itemize}
Appendix~\ref{app:related} reviews related work.
 
\vspace{-5pt}
\section{Component-Property Model}\label{sec:system-model}

\vspace{-5pt}
When faults propagate, downstream anomalies can exceed upstream ones, making score-based ranking unreliable. The Component-Property Model formalizes a class of systems whose properties decompose into internal and external ones, which lets \methodname{} identify root causes without an inter-component dependency graph.

\begin{definition}\label{def:system}
A \textbf{component-based system} is a triple $\mathcal{S} = (\mathcal{C}, \mathcal{I}, \mathcal{E})$ where $\mathcal{C} = \{C_1, \dots, C_n\}$ is a set of $n$ components. $\mathcal{I} = \bigcup_{i=1}^{n} \mathbf{I}_i$ is the set of all internal properties, where $\mathbf{I}_i = \{I_i^1, \dots, I_i^{k_i}\}$ is the set of internal properties of component $C_i$. $\mathcal{E} = \bigcup_{i=1}^{n} \mathbf{E}_i$ is the set of all external properties, where $\mathbf{E}_i = \{E_i^1, \dots, E_i^{m_i}\}$ is the set of external properties of $C_i$.
\end{definition}

\begin{definition}\label{def:internal}
\textbf{Internal properties} $\mathbf{I}_i$ are local to component $C_i$ and not directly observable by other components. Each $I_i^k \in \mathbf{I}_i$ is a real-valued observable (e.g., resource utilization, queue length).
\end{definition}

\begin{definition}\label{def:external}
\textbf{External properties} $\mathbf{E}_i$ of $C_i$ are observable by other components through interactions. Each $E_i^l \in \mathbf{E}_i$ is a real-valued observable. External properties can causally influence external properties of other components, 
e.g, high latency in $C_i$ may cause high latency in $C_j$ if $C_j$ depends on $C_i$.
\end{definition}

\vspace{-5pt}
This component-based system (Figure~\ref{fig:system-structure}), with its internal and external properties, abstracts interconnected systems. Although causal relationships exist within and across components, we assume (1) a directional asymmetry between internal and external properties within each component (Assumption~\ref{ax:direction}), and (2) that inter-component causal influences occur only through external properties (Assumption~\ref{ax:isolation}).

\mypara{Required knowledge.} In addition to the observational data, \methodname{} needs only (i) a mapping from each metric to its component and (ii) an internal or external label for each metric \emph{type}. It does not need the inter-component dependency topology, which makes \methodname{} graph-free. A metric is external if and only if the callers of its component can observe it at the component's interface (latency, error rate, request rate). It is internal if it measures local state that the interface does not expose (CPU, memory, disk I/O, socket count).

\begin{figure}[t!]
\vspace{-5pt}
\centering

\resizebox{\linewidth}{!}{
\begin{tikzpicture}[scale=0.9]
\node[draw, rectangle, minimum width=3cm, minimum height=2.5cm] (C1) at (0,0) {};
    \node at (0, 1.8) {\textbf{Component $C_1$}};
    \node[draw, circle, fill=blue!20] (I11) at (-0.8, 0.8) {$I_1^1$};
    \node[draw, circle, fill=blue!20] (I12) at (0.8, 0.8) {$I_1^2$};
    \node[draw, circle, fill=green!20] (E11) at (0, -0.5) {$E_1^1$};
    \draw[->, thick] (I11) -- (E11);
    \draw[->, thick] (I12) -- (E11);

\node[draw, rectangle, minimum width=3cm, minimum height=2.5cm] (C2) at (5,0) {};
    \node at (5, 1.8) {\textbf{Component $C_2$}};
    \node[draw, circle, fill=blue!20] (I21) at (4.2, 0.8) {$I_2^1$};
    \node[draw, circle, fill=blue!20] (I22) at (5.8, 0.8) {$I_2^2$};
    \node[draw, circle, fill=green!20] (E21) at (5, -0.5) {$E_2^1$};
    \draw[->, thick] (I21) -- (E21);
    \draw[->, thick] (I22) -- (E21);

\draw[->, very thick, red] (E11) -- node[above] {\scriptsize \textit{influence}} (E21);

\node[blue] at (-2.5, 0.8) {Internal};
    \node[green!50!black] at (-2.5, -0.5) {External};
\end{tikzpicture}} \vspace{-10pt}
\caption{Component-Property Model. Solid arrows denote causal influence between \internal{internal properties}~(e.g., configs) and \external{external properties} (e.g., response time), and the red arrow marks inter-component influence. Faults originate in internal properties and propagate through external properties (Assumptions~\ref{ax:direction}--\ref{ax:isolation}).} \label{fig:system-structure}
\vspace{-10pt}
\end{figure}

\begin{assumption}[Diagnostic Sufficiency]\label{asm:sufficiency}
The pre- and post-failure observational data is sufficient to determine the root cause. First, the root cause component is instrumented with at least one internal and one external property. Second, faults manifest as detectable deviations in these properties. Third, no unmeasured common cause (e.g., contention on a shared host) confounds the observed anomalies.
\end{assumption}

\vspace{-5pt}
We term the second requirement \emph{fault manifestation}, and Section~\ref{sec:ranking} quantifies it with anomaly scores and detection thresholds. Unlike~\cite{orchardroot, pham2024baro}, we do not assume the root cause has the highest anomaly score. The challenge is to determine, among all anomalous components, which component is the root cause.

\begin{assumption}[Intracomponent Directionality]\label{ax:direction}
Internal properties are causal ancestors of external properties, and no edge points from an external to an internal property within the same component.
\end{assumption}

\vspace{-5pt}
Assumption~\ref{ax:direction} encodes a directional asymmetry. External properties (e.g., latency) are observable effects of a component's internal states (e.g., resource usage).

\begin{assumption}[Intercomponent Isolation]\label{ax:isolation}
Causal influences between different components occur through external properties. Formally, for any $i \neq j$, there are no direct causal edges from any property of $C_i$ to internal properties of $C_j$.
\end{assumption}

\vspace{-5pt}
In component-based software systems, components are designed to interact only through well-defined interfaces, so one component does not directly alter another component's internal state variables (e.g., resource allocation, configuration). Assumption~\ref{ax:isolation} idealizes this design, and shared hosts, retries, and contention can violate it (Appendix~\ref{app:contention}). Labeling request rate as external also violates Assumption~\ref{ax:direction}, since request rate drives resource usage. Empirically, the internal separation fails in about a third of the RCAEval cases (Appendix~\ref{app:labeling}). Together, Assumptions~\ref{ax:direction}--\ref{ax:isolation} imply that the internal properties of a component have no causal parents outside that component. Appendix~\ref{app:cpm-scope} discusses systems that satisfy these assumptions and gives practical applicability guidance.

\begin{figure*}
\vspace{-5pt}
\centering

\newlength{\waveW}\setlength{\waveW}{0.56cm}      \newlength{\waveH}\setlength{\waveH}{0.24cm}      \newlength{\scoreSize}\setlength{\scoreSize}{0.3cm} \newlength{\rowSep}\setlength{\rowSep}{0.45cm}    \newlength{\colSep}\setlength{\colSep}{0.6cm}     

\colorlet{highColor}{red!60}
\colorlet{highBg}{red!15}
\colorlet{mildColor}{orange!60}
\colorlet{mildBg}{orange!18}
\colorlet{normalColor}{green!50}
\colorlet{normalBg}{green!12}

\newcommand{\drawWaveform}[3][1]{\begin{scope}[shift={(#2)}, scale=#1, transform shape]
        \ifx#3\waveNormal
            \fill[normalBg, rounded corners=1pt] (-0.28,-0.12) rectangle (0.28,0.12);
            \draw[waveform normal] (-0.22, 0)
                sin ++(0.073, 0.03) cos ++(0.073, -0.03)
                sin ++(0.073, -0.03) cos ++(0.073, 0.03)
                sin ++(0.073, 0.03) cos ++(0.073, -0.03);
        \fi
        \ifx#3\waveMild
            \fill[mildBg, rounded corners=1pt] (-0.28,-0.12) rectangle (0.28,0.12);
            \draw[waveform mild] (-0.22, -0.01)
                sin ++(0.06, 0.02) cos ++(0.06, -0.02)
                -- ++(0.02, 0.045)
                sin ++(0.06, 0.02) cos ++(0.06, -0.02)
                sin ++(0.06, 0.02) cos ++(0.06, -0.02);
        \fi
        \ifx#3\waveHigh
            \fill[highBg, rounded corners=1pt] (-0.28,-0.12) rectangle (0.28,0.12);
            \draw[waveform high] (-0.22, -0.02)
                sin ++(0.055, 0.025) cos ++(0.055, -0.025)
                -- ++(0.02, 0.07)
                sin ++(0.055, 0.025) cos ++(0.055, -0.025)
                sin ++(0.055, 0.025) cos ++(0.055, -0.025);
        \fi
    \end{scope}
}
\def\waveNormal{0}
\def\waveMild{1}
\def\waveHigh{2}

\newcommand{\scoreBox}[3]{\fill[#2, rounded corners=2pt] ([shift={(-0.15,-0.15)}]#1) rectangle ([shift={(0.15,0.15)}]#1);
    \ifx&#3&\else
        \node[font=\tiny, white] at (#1) {#3};
    \fi
}

\newcommand{\gridCell}[2]{\fill[#2, rounded corners=2pt] ([shift={(-0.2,-0.15)}]#1) rectangle ([shift={(0.2,0.15)}]#1);
}

\newcommand{\scoringRow}[5]{\node[font=\scriptsize, anchor=east] at ([xshift=-0.85cm]#1) {#2};
    \gridCell{[xshift=-0.6cm]#1}{#3}
    \node[font=\tiny] at ([xshift=-0.32cm]#1) {$\wedge$};
    \gridCell{[xshift=-0.05cm]#1}{#4}
    \draw[tinyarrow] ([xshift=0.18cm]#1) -- ([xshift=0.38cm]#1);
    \gridCell{[xshift=0.6cm]#1}{#5}
}

\newcommand{\rankRow}[5]{\node[font=\scriptsize, anchor=east] at ([xshift=-0.55cm]#1) {#2.};
    \node[font=\scriptsize] at ([xshift=-0.25cm]#1) {#3};
    \fill[#4, rounded corners=2pt] ([xshift=0.05cm, yshift=-0.15cm]#1) rectangle ([xshift=0.45cm, yshift=0.15cm]#1);
    \ifnum#5=1
        \node[font=\scriptsize, green!50!black] at ([xshift=0.7cm]#1) {\checkmark};
    \fi
}

\begin{tikzpicture}[
    scale=0.90,
    transform shape,
mainbox/.style={
        draw, rounded corners=4pt,
        minimum width=2.2cm, minimum height=2.6cm,
        align=center, font=\small
    },
    outbox/.style={
        draw, rounded corners=4pt,
        minimum width=2.4cm, minimum height=2.2cm,
font=\small
    },
ciebox/.style={
        draw, rounded corners=4pt,
        minimum width=8.1cm, minimum height=2.6cm,
        align=center
    },
subbox/.style={
        draw, rounded corners=3pt,
        minimum width=2.3cm, minimum height=1.8cm,
        align=center, font=\small, fill=gray!5
    },
compbox/.style={
        draw, rounded corners=2pt,
        minimum width=1.6cm, inner sep=2pt, font=\scriptsize
    },
arrow/.style={->, thick, >=Stealth},
    subarrow/.style={->, >=Stealth},
    tinyarrow/.style={->, >=Stealth, line width=0.5pt},
boxtitle/.style={font=\small\bfseries},
waveform normal/.style={line width=0.4pt, green!60!black},
    waveform mild/.style={line width=0.4pt, orange!80!black},
    waveform high/.style={line width=0.5pt, red!70!black}
]

\node[mainbox, minimum width=2.4cm] (input) at (0,0) {};
\node[boxtitle, above] at (input.north) {Input};

\node[compbox, minimum height=0.35cm] (c1) at ([xshift=0.2cm, yshift=0.7cm]input.center) {};
\node[compbox, minimum height=0.6cm] (c2) at ([xshift=0.2cm, yshift=0.10cm]input.center) {};
\node[compbox, minimum height=0.9cm] (c3) at ([xshift=0.2cm, yshift=-0.75cm]input.center) {};

\foreach \c/\bx in {1/c1, 2/c2, 3/c3} {
    \node[font=\scriptsize, anchor=east] at (\bx.west) {$C_\c$:};
}

\drawWaveform{[xshift=-0.25cm]c1.center}{\waveHigh}
\drawWaveform{[xshift=0.35cm]c1.center}{\waveMild}

\drawWaveform{[xshift=-0.25cm, yshift=0.14cm]c2.center}{\waveNormal}
\drawWaveform{[xshift=0.35cm, yshift=0.14cm]c2.center}{\waveHigh}
\drawWaveform{[xshift=-0.25cm, yshift=-0.14cm]c2.center}{\waveNormal}
\drawWaveform{[xshift=0.35cm, yshift=-0.14cm]c2.center}{\waveMild}

\drawWaveform{[xshift=-0.25cm, yshift=0.28cm]c3.center}{\waveMild}
\drawWaveform{[xshift= 0.35cm, yshift=0.28cm]c3.center}{\waveNormal}
\drawWaveform{[xshift=-0.25cm, yshift=0cm]   c3.center}{\waveNormal}
\drawWaveform{[xshift= 0.35cm, yshift=0cm]   c3.center}{\waveMild}
\drawWaveform{[xshift=-0.25cm, yshift=-0.28cm]c3.center}{\waveNormal}
\drawWaveform{[xshift= 0.35cm, yshift=-0.28cm]c3.center}{\waveNormal}

\node[font=\scriptsize] at ([xshift=0.0cm, yshift=1.1cm]input.center) {
    \textcolor{gray!70}{Internal} \hspace{1pt} \textcolor{gray!50!black}{External}
};

\node[ciebox, right=0.5cm of input] (cie) {};
\node[boxtitle, above] at (cie.north) {\methodname};

\node[subbox, minimum width=2.1cm, yshift=-0.2cm] (scoring) at ([xshift=-2.75cm]cie.center) {};
\node[subbox, right=0.5cm of scoring, minimum width=1.9cm] (pooling) {};
\node[subbox, right=0.5cm of pooling, minimum width=2.4cm] (ciescore) {};

\node[font=\scriptsize\bfseries, above=0.05cm of scoring.north] {Anomaly Scoring};
\node[font=\scriptsize\bfseries, above=0.05cm of pooling.north] {Score Pooling};
\node[font=\scriptsize\bfseries, above=0.05cm of ciescore.north] {Root Cause Ranking};

\begin{scope}[shift={(scoring.center)}]
\drawWaveform[1.4]{-0.5cm, 0.58cm}{\waveHigh}
    \draw[tinyarrow] (0.0, 0.58cm) -- (0.4, 0.58cm);
    \fill[highColor, rounded corners=2pt] (0.5, 0.43) rectangle (0.9, 0.73);

\drawWaveform[1.4]{-0.5cm, 0.05cm}{\waveMild}
    \draw[tinyarrow] (0.0, 0.05cm) -- (0.4, 0.05cm);
    \fill[mildColor, rounded corners=2pt] (0.5, -0.1) rectangle (0.9, 0.2);

\drawWaveform[1.4]{-0.5cm, -0.48cm}{\waveNormal}
    \draw[tinyarrow] (0.0, -0.48cm) -- (0.4, -0.48cm);
    \fill[normalColor, rounded corners=2pt] (0.5, -0.63) rectangle (0.9, -0.33);

\end{scope}

\begin{scope}[shift={(pooling.center)}]
\node[font=\tiny, gray] at (0, 0.75) {$S^I$};
    \node[font=\tiny, gray] at (0.45, 0.75) {$S^E$};
\foreach \ypos/\lbl/\intc/\extc in {0.45/$C_1$/highColor/mildColor, 0/$C_2$/normalColor/highColor, -0.45/$C_3$/mildColor/mildColor} {
        \node[font=\scriptsize, anchor=east] at (-0.2, \ypos) {\lbl};
        \fill[\intc, rounded corners=2pt] (-0.25, \ypos-0.15) rectangle (0.15, \ypos+0.15);
        \fill[\extc, rounded corners=2pt] (0.25, \ypos-0.15) rectangle (0.65, \ypos+0.15);
    }
\end{scope}

\begin{scope}[shift={(ciescore.center)}]
    \node[font=\tiny, gray] at (0.8, 0.75) {$M$};
\foreach \ypos/\lbl/\lc/\rc/\res in {0.45/$C_1$/highColor/mildColor/highColor, 0/$C_2$/normalColor/highColor/highColor, -0.45/$C_3$/mildColor/mildColor/mildColor} {
        \node[font=\scriptsize, anchor=east] at (-0.7, \ypos) {\lbl};
        \fill[\lc, rounded corners=2pt] (-0.75, \ypos-0.15) rectangle (-0.45, \ypos+0.15);
        \node[font=\tiny] at (-0.32, \ypos) {$\oplus$};
        \fill[\rc, rounded corners=2pt] (-0.2, \ypos-0.15) rectangle (0.1, \ypos+0.15);
        \draw[tinyarrow] (0.18, \ypos) -- (0.58, \ypos);
        \fill[\res, rounded corners=2pt] (0.65, \ypos-0.15) rectangle (0.95, \ypos+0.15);
    }
\end{scope}

\draw[subarrow] (scoring.east) -- (pooling.west);
\draw[subarrow] (pooling.east) -- (ciescore.west);

\node[outbox, right=0.5cm of cie, yshift=0.2cm] (output) {};
\node[boxtitle, above] at (output.north) {Output};

\begin{scope}[shift={(output.center)}]
    \foreach \ypos/\rank/\lbl/\col/\chk in {0.2/1/$C_1$/highColor/1, -0.3/2/$C_2$/highColor/0, -0.8/3/$C_3$/mildColor/0} {
        \node[font=\scriptsize, anchor=east] at (-0.25, \ypos) {Top-\rank:};
        \node[font=\scriptsize] at (-0.05, \ypos) {\lbl};
        \fill[\col, rounded corners=2pt] (0.25, \ypos-0.15) rectangle (0.65, \ypos+0.15);
        \ifnum\chk=1
            \node[font=\scriptsize, green!50!black] at (0.9, \ypos) {\checkmark};
        \fi
    }
\end{scope}
\node[font=\scriptsize\bfseries] at ([yshift=0.8cm]output.center) {Ranked List};

\shade[left color=green!60, right color=red!70, rounded corners=1pt]
    ([xshift=-0.85cm, yshift=-1.35cm]output.center) rectangle ([xshift=0.85cm, yshift=-1.25cm]output.center);
\node[font=\tiny] at ([yshift=-1.5cm]output.center) {Low $\leftarrow$ Score $\rightarrow$ High};

\draw[arrow] (input.east) -- (cie.west);
\draw[arrow] ([yshift=0.1cm]cie.east) -- ([yshift=-0.1cm]output.west);

\end{tikzpicture} \vspace{-5pt}
\caption{Overview of \methodname{}. Given observational data of internal and external properties for each component, \methodname{} first computes deviation-based anomaly scores, then pools property-level scores into component-level internal ($S^I_i$) and external ($S^E_i$) scores, and finally ranks components. The root cause ($C_1$) is anomalous in both internal and external properties, whereas the affected $C_2$ has a larger external anomaly but a nominal internal one.}
\label{fig:overview}
\vspace{-10pt}
\end{figure*}

\vspace{-5pt}
\section{\methodname{}: Our Proposed Approach}\label{sec:cie}

\methodname{} rests on a simple principle. \emph{Under the Component-Property Model, the root cause exhibits substantial anomalies in both internal and external properties, while downstream affected components show only external anomalies due to fault propagation} (see Figure~\ref{fig:overview}).
\vspace{-5pt}

\subsection{Anomaly Scoring}\label{subsec:anomaly_scoring}

Anomaly scoring is foundational to RCA, because every later step depends on how deviations from normal behavior are quantified. An anomaly scorer is a function $S: \mathcal{X} \to \mathbb{R}_{\geq 0}$ on the value domain $\mathcal{X} \subseteq \mathbb{R}$ of a property. We adopt \emph{deviation-based anomaly scoring}, which directly measures how far an observation departs from a reference baseline.

\begin{definition}\label{def:deviation-score}
A \emph{deviation-based anomaly scorer} is a function \(
    S(x; \theta) = |x - c(\theta)|/s(\theta)
\),
where the location $c(\theta)$ and the scale $s(\theta) > 0$ are estimated from reference data $\theta$. Examples include the z-score~\cite{orchardroot} and the interquartile range (IQR)-based score~\cite{pham2024baro}.
\end{definition}

For each property $P$ with values $x_t$, the reference data $\theta$ is the pre-failure window only. The z-score uses the mean and standard deviation of the pre-failure values as $c$ and $s$, and the IQR-based score uses their median and IQR. A zero scale is replaced by $1$, and small nonzero scales are used as they are. The property-level score applies $S$ to the peak post-failure value, $S(P) = |\max_{t \in \mathcal{T}_{\text{post}}} x_t - c|/s$, where $\mathcal{T}_{\text{post}}$ is the post-failure window. Because the peak is taken before the absolute value, the score targets upward excursions such as latency or resource spikes, and a window lying entirely below $c$ is scored by its smallest drop. The conditions below therefore apply to the peak value, not in the deviation of the whole trajectory.

For RCA, a scorer must preserve differences in deviation magnitude. \emph{Monotonicity} ensures that larger deviations yield larger scores, and \emph{injectivity} ensures that distinct deviations receive distinct scores, preventing ties that make ranking impossible.

\begin{definition}[Monotonicity]\label{def:monotonicity}
A scorer $S$ is \emph{monotonic} if larger deviations yield larger scores, \(|x - c| > |y - c| \implies S(x) > S(y)\).
\end{definition}

\begin{definition}[Injectivity]\label{def:injectivity}
A scorer $S$ is \emph{injective} if it assigns distinct scores to observations with distinct deviations, \(|x - c| \neq |y - c| \implies S(x) \neq S(y)\).
\end{definition}

\begin{lemma}\label{lem:mono-inject}
Monotonicity (Definition~\ref{def:monotonicity}) implies injectivity (Definition~\ref{def:injectivity}). The converse fails, e.g., $S(x) = 1/(1+|x-c|)$ is injective but not monotonic.
\end{lemma}

\begin{theorem}[Deviation-Based Scorers Are Monotonic and Injective]\label{thm:deviation-conditions}
Every deviation-based anomaly score (Definition~\ref{def:deviation-score}) is monotonic and injective.
\end{theorem}

Both conditions follow from $S$ being a strictly increasing function of $|x - c|$ (Appendix~\ref{app:proofs}). We now show that injectivity is necessary.

\begin{theorem}\label{thm:rca-necessity}
If a scorer $S$ is not injective, there exist two failure scenarios with identical scores but different root causes, where each root cause has the strictly largest deviation. Any ranking rule that depends on the data only through $S$ fails in at least one of them.
\end{theorem}

A scorer that maps distinct deviations to equal scores discards exactly the information that separates the two scenarios (proof in Appendix~\ref{app:proofs}). Information-Theoretic (IT) scorers violate injectivity when post-failure observations exceed the reference range (Appendix~\ref{subsec:it-saturation}).

\vspace{-5pt}
\subsection{Score Pooling}\label{sec:pooling}

\vspace{-5pt}
Components expose different numbers of properties depending on their instrumentation, and a component with many properties should not outrank one with few merely for that reason. We therefore pool property-level scores into a single internal score $S^I_i$ and a single external score $S^E_i$ for each component $C_i$.

\begin{definition}[Score Pooling]\label{def:pooling}
A \emph{pooling function} $\phi: \bigcup_{m \geq 1} \mathbb{R}_{\geq 0}^m \to \mathbb{R}_{\geq 0}$ aggregates property-level anomaly scores into a component-level score. For each component $C_i$, $S^T_i = \phi\bigl(\{S(P) : P \in \mathbf{T}_i\}\bigr)$ for $T \in \{I, E\}$, with $\phi(\emptyset) = 0$.
\end{definition}
We require $\phi$ to be nondecreasing in each argument (weak monotonicity), so that anomaly evidence is never suppressed, and permutation-invariant, i.e., $\phi(x_1, \ldots, x_m) = \phi(x_{\pi(1)}, \ldots, x_{\pi(m)})$ for any permutation $\pi$. The aggregators $\max$, $\mathrm{sum}$, and $\mathrm{mean}$ satisfy both conditions but differ in size invariance. Adding nominal properties (scores below the current maximum) leaves $\max$ unchanged, can only increase $\mathrm{sum}$, and dilutes $\mathrm{mean}$. We use $\max$ by default, and Section~\ref{subsec:robustness} shows that \methodname{} is robust across these choices.

\vspace{-5pt}
\subsection{Root Cause Ranking}\label{sec:ranking}

\vspace{-5pt}
Given pooled scores $S^I_i$ and $S^E_i$ for each component $C_i$, we rank components to identify the root cause. A combiner $M: \mathbb{R}_{\geq 0}^2 \to \mathbb{R}_{\geq 0}$ maps $(S^I_i, S^E_i)$ to a root cause score, written $M(C_i)$. We require $M$ to be nondecreasing in each argument (weak monotonicity), so that stronger evidence never lowers the score. Our guarantees rest on \emph{internal boundedness}, which the conjunctive score always satisfies and the additive score only under bounded external amplification.

\begin{definition}[Internal Boundedness]\label{def:internal-bound}
A combiner $M$ is \emph{internally bounded} on $D \subseteq \mathbb{R}_{\geq 0}^2$ if there exists a nondecreasing $f: \mathbb{R}_{\geq 0} \to \mathbb{R}_{\geq 0}$ such that $M(s^I, s^E) \leq f(s^I)$ for all $(s^I, s^E) \in D$. We omit $D$ when $D = \mathbb{R}_{\geq 0}^2$.
\end{definition}

\vspace{-5pt}
Internal boundedness caps the score of a component whose internal score is nominal, however amplified its external score. It is not vacuous, since $M = \max(s^I, s^E)$ and the external-only score $M = s^E$ violate it. A cap alone does not ensure that the root cause scores above it (see Remark~\ref{rem:bound-not-sufficient} in Appendix~\ref{app:proofs}), so we also use the following condition.

\begin{definition}[Threshold Separation]\label{def:threshold-sep}
An internally bounded $M$ with bounding function $f$ satisfies \emph{threshold separation} at level $\epsilon \geq 0$ if $M(s^I, s^E) > f(\epsilon)$ whenever $s^I > \epsilon$ and $s^E > \epsilon$.
\end{definition}

\vspace{-5pt}
\mypara{Additive Score.}
Our default treats internal and external anomalies as cumulative evidence:
\begin{equation}\label{eq:add-score}
M_{\text{add}}(C_i) = S^I_i + S^E_i - \log(1 + S^I_i + S^E_i),
\end{equation}
i.e., $M_{\text{add}}(C_i) = g(S^I_i + S^E_i)$ with $g(t) = t - \log(1 + t)$ strictly increasing on $\mathbb{R}_{\geq 0}$, so $M_{\text{add}}$ ranks components exactly as $S^I_i + S^E_i$ does.

\begin{proposition}\label{prop:add-valid}
$M_{\text{add}}$ is weakly monotone but not internally bounded on $\mathbb{R}_{\geq 0}^2$. Under bounded external amplification with constants $\gamma, \delta \geq 0$, i.e., on $D_{\gamma,\delta} = \{(s^I, s^E) : s^E \leq \gamma s^I + \delta\}$, it is internally bounded with $f(s) = g((1+\gamma)s + \delta)$, and $M_{\text{add}}(s^I, s^E) > f(\epsilon)$ if and only if $s^I + s^E > (1+\gamma)\epsilon + \delta$.
\end{proposition}

\vspace{-5pt}
\mypara{Conjunctive Score.}
An alternative requires \emph{both} scores to be high by taking their minimum:
\begin{equation}\label{eq:conj-score}
M_{\text{conj}}(C_i) = \min(S^I_i, S^E_i).
\end{equation}
This stance is more conservative and may underweight components with asymmetric anomaly patterns.

\begin{proposition}\label{prop:conj-valid}
$M_{\text{conj}}$ is weakly monotone (but not strictly increasing), internally bounded with $f(s) = s$ for $s \geq 0$, and satisfies threshold separation at every level $\epsilon \geq 0$.
\end{proposition}

\vspace{-5pt}
\mypara{Theoretical Guarantee.}
Let $C_r$ be the root cause with internal properties $\mathbf{I}_r$ (Definition~\ref{def:system}), and let the fault be an intervention on $I_r^j$, the $j$-th element of $\mathbf{I}_r$. Network faults, which originate outside component internals (240 of the 735 evaluated cases), fall outside this fault model. We compare $C_r$ with every non-root component $C_a$, $a \neq r$, affected or not. For $\alpha_a \in (0,1)$, let $\epsilon_a$ be the $(1-\alpha_a)$-quantile of the null distribution of $S^I(C_a)$, i.e., of $C_a$'s pooled internal statistic computed with the same window lengths in fault-free operation.

\begin{lemma}\label{lem:internal-separation}
Let Assumptions~\ref{asm:sufficiency}--\ref{ax:isolation} hold, and let the internal properties of each $C_a$ be stationary, i.e., absent an intervention on $C_a$, they have the same distribution before and after the fault. Then:
\vspace{-5pt}
\begin{enumerate}[leftmargin=*,nolistsep]
    \item If $\phi(x) \geq \max_i x_i$ (e.g., $\max$ and $\mathrm{sum}$, but not $\mathrm{mean}$), then $S^I(C_r) \geq S(I_r^j)$.
    \item For any $\phi$ and $a \neq r$, $P(S^I(C_a) \leq \epsilon_a) \geq 1 - \alpha_a$.
    \item With probability at least $1 - \sum_{a \neq r} \alpha_a$, $S^I(C_a) \leq \epsilon_a$ for all $a \neq r$.
\end{enumerate}
\end{lemma}
\vspace{-5pt}
Part~2 holds because $C_a$'s internal properties have no causal parents outside $\mathbf{I}_a$ (Assumptions~\ref{ax:direction}--\ref{ax:isolation}) and no unmeasured confounder links them to the fault (Assumption~\ref{asm:sufficiency}). The fault is thus not an intervention on $C_a$, so stationarity applies. Part~3 follows from a union bound over $a \neq r$. The thresholds serve only the analysis, and \methodname{} never computes them.

\begin{theorem}\label{thm:cie-correct}
Fix $\epsilon \geq 0$ and suppose the separation event $S^I(C_a) \leq \epsilon$ for every $a \neq r$ holds.
\vspace{-5pt}
\begin{enumerate}[leftmargin=*,nolistsep]
    \item Let $M$ be weakly monotone, internally bounded with $f$, and threshold-separating at $\epsilon$. If $S^I(C_r) > \epsilon$ and $S^E(C_r) > \epsilon$, then $M(C_r) > M(C_a)$ for every $a \neq r$. This covers $M_{\text{conj}}$ with no bound on $S^E(C_a)$.
    \item If $S^E(C_a) \leq \gamma S^I(C_a) + \delta$ for every $a \neq r$ and $S^I(C_r) + S^E(C_r) > (1+\gamma)\epsilon + \delta$, then $M_{\text{add}}(C_r) > M_{\text{add}}(C_a)$ for every $a \neq r$.
\end{enumerate}
\end{theorem}
\vspace{-5pt}
In both parts, $M(C_a) \leq f(S^I(C_a)) \leq f(\epsilon) < M(C_r)$ (Appendix~\ref{app:proofs}). Applying Lemma~\ref{lem:internal-separation} at $\epsilon^* = \max_{a \neq r} \epsilon_a$ then gives a probabilistic guarantee.

\begin{corollary}\label{cor:prob-first}
Under the conditions of Lemma~\ref{lem:internal-separation}, let $R$ be the event that both scores of the root cause exceed $\epsilon^*$ (manifestation). Then $M_{\text{conj}}$ ranks $C_r$ first with probability at least $1 - \sum_{a \neq r} \alpha_a - P(R^c)$. $M_{\text{add}}$ does so with probability at least $1 - \sum_{a \neq r} \alpha_a - P(G^c) - P(R'^c)$, where $G$ is the event that bounded amplification holds at every $a \neq r$ and $R'$ the event $S^I(C_r) + S^E(C_r) > (1+\gamma)\epsilon^* + \delta$.
\end{corollary}

\vspace{-5pt}
The probability is over the observed trajectories, with the fault and the thresholds fixed. If the root cause manifests surely ($P(R^c) = 0$), choosing $\alpha_a = \alpha/(n-1)$ gives at least $1-\alpha$ for $M_{\text{conj}}$. The conjunctive score is thus guaranteed on the separation event whenever the root cause manifests, with no bound on external amplification. The additive score is guaranteed only under bounded amplification at non-root components and a quantified root-cause separation, but it ranks better empirically (Section~\ref{subsec:robustness}).

\vspace{-5pt}
\mypara{Why Not Rank by Internal Score Alone?}
On the separation event, ranking by $S^I$ alone already succeeds whenever $S^I(C_r) > \epsilon$, a weaker condition than that of Part~1. Part~2 requires only the total $S^I(C_r) + S^E(C_r)$ to exceed $(1+\gamma)\epsilon + \delta$, so strong external evidence can compensate for a weak internal anomaly, e.g., a code bug without a resource spike. Outside the separation event, the external score can also suppress unaffected components whose internal metrics fluctuate. Any such benefit is empirical, as no theorem here covers it.
 
\section{Evaluation}\label{sec:evaluation}

We answer four research questions. \textbf{RQ1:} How effective is \methodname{}? \textbf{RQ2:} How efficient is \methodname{}? \textbf{RQ3:} How robust is \methodname{}? \textbf{RQ4:} How does \methodname{} compare with ranking by marginal, internal-only, or external-only scores?

\vspace{-5pt}
\subsection{Experimental Settings}\label{subsec:settings}

\mypara{Datasets.} We evaluate \methodname{} on nine datasets from the RCAEval benchmark~\cite{pham2025rcaeval}, comprising 735 fault-injection cases. The datasets are collected from three microservice systems (Online Boutique~\cite{onlineboutique}, Sock Shop~\cite{sockshop}, and Train Ticket~\cite{zhou2021trainticket}), containing 12 to 64 interacting components. They cover 11 fault types (4~resource faults, 2~network faults, and 5~code-level faults), and each case is annotated with its root cause component. \methodname{} uses only the metrics, which all datasets contain, and treats each metric as one property of its component. All methods split each case at its benchmark fault-injection time into a pre-failure and a post-failure window (about 6 to 35 minutes each, at one-second resolution). Table~\ref{tab:datasets} summarizes the dataset statistics.

\begin{table}[t]
\centering
\caption{Statistics of the nine RCAEval benchmark datasets (\#Properties: metrics, \#Faults: fault types).} \label{tab:datasets}
\vspace{-5pt}
\resizebox{0.9\columnwidth}{!}{\setlength{\tabcolsep}{4pt}
\begin{tabular}{lcccc}
\toprule
System (RE1/RE2/RE3) & \#Components & \#Properties & \#Faults & \#Cases \\
\midrule
Online Boutique (OB) & 12 & 49/72/69 & 5/6/5 & 125/90/30 \\
Sock Shop (SS) & 13 & 59/75/80 & 5/6/4 & 125/90/30 \\
Train Ticket (TT) & 64 & 220/369/315 & 5/6/4 & 125/90/30 \\
\midrule
\textbf{Total} & 12--64 & 49--369 & 11 & 735 \\
\bottomrule
\end{tabular}
 }\vspace{-4pt}
\end{table}

\mypara{Internal--External Property Decomposition.} We classify properties based on their observability to other components~\cite{li2022causal,googlesre}. External properties are observable metrics at component boundaries, namely response time (latency), request rate, and error rate. Internal properties are local resource states, namely CPU usage, memory utilization, disk I/O, and socket count. We examine this classification empirically in Appendix~\ref{app:labeling} and quantify performance when non-root internal anomalies violate isolation in Appendix~\ref{app:contention}.

\mypara{Baselines.} 
We compare \methodname{} against twelve state-of-the-art RCA methods. Eight of them do not require structural knowledge. BARO~\cite{pham2024baro}, Score-Ordering~\cite{orchardroot}, and Cholesky~\cite{li2025root} rank by anomaly scores or linear structural equations, while PC-PageRank~\cite{pham2025rcaeval}, RCD~\cite{ikram2022root}, BRCD~\cite{lee2026brcd}, CIRCA~\cite{li2022causal}, and DoWhy-GCM~\cite{dowhygcm2024} learn a causal graph from the data of each case. As published, CIRCA and DoWhy-GCM take a known graph, but we give them a PC-learned one, as RCAEval does for CIRCA. The other four require structural knowledge, namely Counterfactual~\cite{hardt2024petshop}, RCLAgent~\cite{zhang2025adaptive}, Simple-Traversal~\cite{orchardroot}, and Smooth-Traversal~\cite{orchardroot}. As RCAEval does not provide dependency graphs, we obtain them for this group by instrumenting the deployed systems. RCLAgent also requires traces, and we evaluate it on the 239 traced cases for which it returns a ranking. We run RCD, BRCD, CIRCA, and DoWhy-GCM separately with a time limit of one to four hours per case (Table~\ref{tab:rca_rcaeval}). Appendix~\ref{app:baselines} describes the baselines, their settings, and the graph collection.

\mypara{Evaluation Metrics.} We use two standard metrics, Top-$k$ and Avg@$k$. Given a set of failure cases $\mathcal{Q}$,
\begin{equation}
    \text{Top-}k = \frac{1}{|\mathcal{Q}|} \sum\nolimits_{q\in \mathcal{Q}} \mathbb{1}\bigl[\exists\, i \leq k : R^q[i] \in V^q_{rc}\bigr],
\end{equation}
where $R^q[i]$ is the $i$th result that an RCA method ranks for case $q$ (a ranked property counts as its component), and $V^q_{rc}$ is the true root cause set of $q$. $\text{Avg}@k = \frac{1}{k}\sum_{j=1}^k \text{Top-}j$ summarizes the overall RCA performance. We compute each metric per fault type and report its macro average over the fault types (the 11 fault types for overall results, and the fault types of a dataset for per-dataset results).

To assess significance, we use the paired Wilcoxon signed-rank test~\cite{wilcoxon1945individual,mcdonald2014wilcoxon} on the per-case Avg@5. When comparing \methodname{} against multiple baselines, we apply the Holm-Bonferroni correction~\cite{holm1979simple} at a family-wise level of $0.05$. In our tables, (+) indicates that \methodname{} performs significantly better than the baseline, (-) significantly worse, and {\scriptsize $(\approx)$} no significant difference.

\begin{table}
\vspace{-5pt}
\centering
\caption{Overall performance.}
\vspace{-5pt}
\label{tab:rca_rcaeval}
\resizebox{\columnwidth}{!}{\begin{tabular}{lcccc}
\toprule
Method & \multicolumn{3}{c}{AVERAGE} & \multicolumn{1}{c}{STAT} \\
\cline{2-4} \cline{5-5}
 & Top-1 & Top-3 & Avg@5 & p-value \\
\midrule
Counterfactual & 0.07 & 0.12 & 0.12 & $1.8\times10^{-112}${\scriptsize (+)} \\
RCLAgent$^\ast$ & 0.10 & 0.18 & 0.16 & $1.1\times10^{-38}${\scriptsize (+)} \\
Smooth-Traversal & 0.05 & 0.19 & 0.16 & $1.9\times10^{-113}${\scriptsize (+)} \\
PC-PageRank & 0.09 & 0.21 & 0.21 & $6.5\times10^{-110}${\scriptsize (+)} \\
Cholesky & 0.09 & 0.22 & 0.22 & $3.5\times10^{-107}${\scriptsize (+)} \\
Simple-Traversal & 0.09 & 0.24 & 0.22 & $5.6\times10^{-108}${\scriptsize (+)} \\
Score-Ordering & 0.11 & 0.30 & 0.29 & $2.0\times10^{-100}${\scriptsize (+)} \\
RCD & 0.35 & 0.45 & 0.45 & $1.5\times10^{-65}${\scriptsize (+)} \\
BRCD & 0.29 & 0.58 & 0.53 & $7.9\times10^{-53}${\scriptsize (+)} \\
DoWhy-GCM & 0.49 & 0.55 & 0.55 & $4.4\times10^{-50}${\scriptsize (+)} \\
BARO & 0.19 & 0.74 & 0.63 & $3.3\times10^{-61}${\scriptsize (+)} \\
CIRCA & \underline{0.56} & \underline{0.87} & \underline{0.82} & $1.2\times10^{-15}${\scriptsize (+)} \\
\midrule
\textbf{\methodname} & \textbf{0.68} & \textbf{0.91} & \textbf{0.87} & N/A \\
\bottomrule
\multicolumn{5}{l}{\footnotesize $^\ast$Evaluated on the 239 cases with traces.}\\
\end{tabular}
 }\vspace{-7pt}
\end{table}

\subsection{Effectiveness in RCA}\label{subsec:effectiveness}

For RQ1, Table~\ref{tab:rca_rcaeval} presents the overall results. Due to space limitations, we put the detailed results for each dataset, including RE2OB (Table~\ref{tab:rca_re2ob}), in Appendix~\ref{appendix:rca-effectiveness}. For example, on RE2SS (Table~\ref{tab:rca_re2ss}), \methodname{} has the highest Top-1 accuracy on all six fault types, including $80\%$ on packet loss, where the next-best method (DoWhy-GCM) reaches $56\%$.
\mypara{Overall Performance.}
The best baseline, CIRCA, reaches $56\%$ Top-1, $87\%$ Top-3, and $82\%$ Avg@5, against $68\%$, $91\%$, and $87\%$ for \methodname{}, a relative Top-1 improvement of $21\%$. Relative to BARO, which ranks properties by their marginal anomaly scores, the Top-1 improvement is $258\%$ (from $19\%$ to $68\%$). The Avg@5 improvements over all twelve baselines are statistically significant. Per dataset, \methodname{} has the highest Top-1 accuracy on six of the nine datasets, ties CIRCA on RE1TT and RE3OB, and trails CIRCA and DoWhy-GCM on RE3SS ($20\%$ vs.\ $37\%$ and $25\%$, 30 cases). The Avg@5 differences from CIRCA on these three datasets are not significant.

\begin{tcolorbox}[colormetabox]
\textbf{Finding 1:} \methodname{} achieves \textbf{68\% Top-1 accuracy}, a 21\% relative improvement over the best of the twelve evaluated baselines (CIRCA).
\end{tcolorbox}

\mypara{Performance by Fault Type.}
On RE2OB (Online Boutique with multimodal data, Table~\ref{tab:rca_re2ob} in Appendix~\ref{appendix:rca-effectiveness}), \methodname{} achieves $90\%$ Top-1 and $96\%$ Top-3 accuracy. For resource faults, \methodname{} achieves perfect accuracy of $100\%$ Top-3 across CPU, MEM, DISK, and SOCKET faults, as does CIRCA, but \methodname{} ranks the root cause first more often ($87$--$100\%$ vs.\ $47$--$73\%$ Top-1). These faults directly corrupt internal properties (resource metrics), which \methodname{}'s internal--external decomposition is designed to detect. Network faults (DELAY, LOSS) are more challenging. \methodname{} achieves $80\%$ Top-1 on both, which ties CIRCA on DELAY and exceeds the next-best method on LOSS (DoWhy-GCM, $73\%$). Its Top-3 accuracy is $93\%$ on LOSS, the highest of all methods, but $80\%$ on DELAY, below that of CIRCA ($87\%$), possibly because network delays need not manifest immediately in the internal resource metrics on which \methodname{} relies. 

\begin{tcolorbox}[colormetabox]
\textbf{Finding 2:} On RE2OB, \methodname{} achieves $100\%$ Top-3 on resource faults (CPU, MEM, DISK, SOCKET), which corrupt internal properties, and $80\%$ Top-1 on network faults (DELAY, LOSS).
\end{tcolorbox}

\mypara{Where Do Baselines Fall Short?}
IT-score-based methods (Score-Ordering, Smooth-Traversal) suffer from \emph{saturation}, which gives identical scores to all components whose anomalies exceed historical bounds (Appendix~\ref{subsec:it-saturation}). BARO ranks individual properties by an IQR-scaled score. A near-zero pre-failure IQR can inflate the score of an unrelated service, which BARO then ranks first (see the failure case analysis below). Property-level ranking can also fail when downstream effects exhibit larger external anomalies than the root cause (Theorem~\ref{thm:counter}). BARO's strong Top-3 ($74\%$) but weak Top-1 ($19\%$) is consistent with both effects. Cholesky, which relies on linear structural equations, and the graph-based PC-PageRank achieve only $9\%$ Top-1 accuracy. CIRCA, which tests each metric against its parents in a learned graph, trails \methodname{} most on disk and socket faults ($36\%$ vs.\ $66\%$ and $38\%$ vs.\ $76\%$ Top-1). RCLAgent, although it uses topology knowledge, achieves only $10\%$ Top-1 on the cases with traces, possibly because its bounded recursive exploration covers only part of the trace data.

\begin{table}[t]
\vspace{-5pt}
\centering
\caption{BARO's root cause ranking for a disk fault in \service{currencyservice}. The root cause is at ranks 2 and 4.} \label{tab:case-study}
\vspace{-3pt}
\resizebox{0.75\columnwidth}{!}{\begin{tabular}{clr}
\toprule
Rank & Root Cause Candidate & Anomaly Score \\
\midrule
1 & \service{emailservice}-\external{latency} & $4.6 \times 10^{14}$ \\
2 & \service{currencyservice}-\internal{diskio} & $5.1 \times 10^{9}$ \\
3 & \service{redis}-\internal{diskio} & $5.0 \times 10^{4}$ \\
4 & \service{currencyservice}-\external{latency} & $353.6$ \\
\bottomrule
\end{tabular}
}\vspace{-7pt}
\end{table}

\begin{tcolorbox}[colormetabox]
\textbf{Finding 3:} Baselines fall short through IT-score saturation, which ties anomalous components, and marginal-score ranking of properties, which a propagated effect or a poorly scaled score can dominate.
\end{tcolorbox}

\mypara{Failure Case Analysis.}
We analyze a disk fault in \service{currencyservice} of the Online Boutique system (RE2OB) to illustrate how score-based methods can rank a non-root property above the true root cause. Table~\ref{tab:case-study} shows BARO's top-4 ranked candidates. The highest-scoring candidate is \service{emailservice}-\external{latency}, an external property of another service, with score $4.6 \times 10^{14}$. This extreme score arises because the pre-failure latency of \service{emailservice} is nearly constant (its IQR is close to zero), so a small post-failure increase (from $4.6$ to $5.0$~ms) yields a huge IQR-normalized deviation. The true root cause (\service{currencyservice}) appears at both rank 2 (\internal{diskio}) and rank 4 (\external{latency}), exhibiting anomalies in both internal and external properties.

This analysis illustrates two points. First, \emph{the highest anomaly score need not indicate the root cause}, i.e., the top candidate may be a non-root property. Second, \emph{the scale estimate matters}. Under the z-score, both scores of \service{emailservice} stay low ($S^I = 3.5$, $S^E = 14.3$), and \methodname{} ranks \service{currencyservice} first ($S^I = 141.1$, $S^E = 473.7$, $M_{\text{add}} = 608.4$), while \service{emailservice} ranks only sixth ($M_{\text{add}} = 14.9$). In this case, the root cause also has the largest z-score of any property, so a marginal z-score ranking would also succeed. The fix comes from the scorer, not the internal--external decomposition.

\vspace{-5pt}
\subsection{Efficiency in RCA}\label{subsec:efficiency}
\vspace{-5pt}
For RQ2, we measure each method's running time on every case it processes, since RCA must complete quickly to enable a timely response. Figure~\ref{fig:efficiency-boxplot} shows the runtime distribution per case.

\mypara{Runtime Performance.}
\methodname{} achieves a median inference time of $8.2$\,ms per case, making it the second fastest after BARO ($6.3$\,ms). Both methods complete most diagnoses in under $15$\,ms. Methods that construct causal graphs or run large language model (LLM) inference are much slower. PC-PageRank needs a median of $1.2$\,s for causal discovery, and RCLAgent $79.3$\,s for LLM inference. \methodname{} is thus $146\times$ faster than PC-PageRank and about $9{,}700\times$ faster than RCLAgent. We timed BRCD, CIRCA, RCD, and DoWhy-GCM on shared cluster nodes (4 cores, 2 for CIRCA), so Figure~\ref{fig:efficiency-boxplot} omits them. Their median times per case are $9.9$\,s, $13.6$\,s, $18.1$\,s, and $708$\,s, respectively (BRCD and DoWhy-GCM on the cases they complete).

\begin{figure}
\centering
\includegraphics[width=\columnwidth]{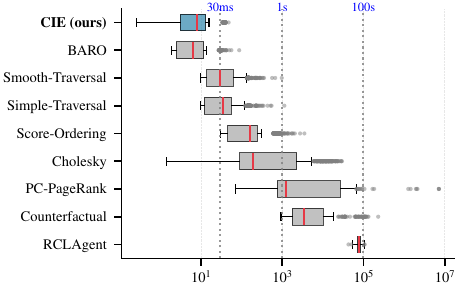}
\caption{Running time per failure case.} \label{fig:efficiency-boxplot}
\end{figure}

\begin{tcolorbox}[colormetabox]
\textbf{Finding 4:} \methodname{} requires a median of $8.2$\,ms per diagnosis, $30\%$ more than the fastest baseline (BARO), whereas the best baseline (CIRCA) requires $13.6$\,s on a 2-core cluster node.
\end{tcolorbox}

\subsection{Robustness Analysis}\label{subsec:robustness}
To answer RQ3, we examine the robustness of \methodname{} to its design choices (Table~\ref{tab:robustness}) and its sensitivity to the length of the post-failure data.

\begin{table}[t]
\centering
\caption{Robustness to design choices ($^\dagger$default).}
\label{tab:robustness}
\resizebox{\columnwidth}{!}{\begin{tabular}{llccccc}
\toprule
Category & Variant & Top-1 & Top-3 & Top-5 & Avg@5 & Time (ms) \\
\midrule
\multirow{4}{*}{\shortstack{\textit{Anomaly}\\ \textit{Scoring}}}
  & z-score$^\dagger$ & \textbf{0.68} & \textbf{0.91} & \textbf{0.98} & \textbf{0.87} & 12.6 \\
  & IQR & \underline{0.65} & \underline{0.80} & \underline{0.85} & \underline{0.78} & 16.2 \\ & IT-z-score & 0.14 & 0.51 & 0.75 & 0.48 & 35.9 \\
  & IT-IQR & 0.13 & 0.48 & 0.72 & 0.45 & 35.7 \\
\midrule
\multirow{3}{*}{\shortstack{\textit{Score}\\ \textit{Pooling}}}
  & max$^\dagger$ & \underline{0.68} & 0.91 & \textbf{0.98} & 0.87 & 12.6 \\
  & sum & \textbf{0.69} & \textbf{0.93} & \textbf{0.98} & \textbf{0.89} & 13.2 \\
  & mean & \underline{0.68} & \underline{0.92} & \textbf{0.98} & \underline{0.88} & 12.8 \\
\midrule
\multirow{2}{*}{\shortstack{\textit{Root Cause}\\ \textit{Ranking}}}
  & conjunctive & \underline{0.66} & \underline{0.80} & \underline{0.86} & \underline{0.77} & 12.6 \\
  & additive$^\dagger$ & \textbf{0.68} & \textbf{0.91} & \textbf{0.98} & \textbf{0.87} & 12.6 \\ \bottomrule
\end{tabular}
}\vspace{-5pt}
\end{table}

\mypara{Different Anomaly Scorers.}
We evaluate four anomaly scorers: z-score and IQR-based score (deviation-based, satisfying Theorem~\ref{thm:deviation-conditions}), and their IT-wrapped variants (IT-z-score, IT-IQR). The z-score reaches $68\%$ Top-1, whereas the IT-wrapped variants reach only $13$--$14\%$. Under saturation, IT scores assign the same maximal value to every property that exceeds its reference maximum (Theorem~\ref{thm:it-saturation}). Saturated components then tie, and under max pooling the root cause is ranked first with probability at most $1/m$ among $m$ tied components (Corollary~\ref{cor:it-fails}).

\begin{tcolorbox}[colormetabox]
\textbf{Finding 5:} The z-score far outperforms the IT-based scorers ($68\%$ vs.\ $13$--$14\%$ Top-1), consistent with the ties that saturation creates (Corollary~\ref{cor:it-fails}).
\end{tcolorbox}

\mypara{Different Pooling Strategies.}
The three pooling functions $\max$ (worst deviation), $\text{sum}$ (total anomaly evidence), and $\text{mean}$ (average deviation) achieve comparable performance ($68$--$69\%$ Top-1). This is consistent with Lemma~\ref{lem:internal-separation}, whose non-root separation (Part~2) holds for any pooling function. Only $\max$ and $\text{sum}$ also carry Part~1, yet $\text{mean}$ loses no accuracy on RCAEval.

\begin{tcolorbox}[colormetabox]
\textbf{Finding 6:} \methodname{} is robust to the three tested pooling functions. The common nondecreasing, permutation-invariant aggregators $\max$, $\text{sum}$, and $\text{mean}$ yield comparable performance.
\end{tcolorbox}

\mypara{Root Cause Ranking.}
Additive scoring $M_{\text{add}}$ (default, Eq.~\ref{eq:add-score}) outperforms conjunctive scoring $M_{\text{conj}}$ (Eq.~\ref{eq:conj-score}), with $68\%$ vs.\ $66\%$ Top-1 and $91\%$ vs.\ $80\%$ Top-3. The sum likely lets strong external evidence compensate for a weak internal anomaly, which the minimum discards, so the two scores trade accuracy for guarantees (Section~\ref{sec:ranking}).

\begin{tcolorbox}[colormetabox]
\textbf{Finding 7:} Both scores have similar Top-1 accuracy ($66$--$68\%$). The additive score is better at Top-3 ($91\%$ vs.\ $80\%$) and Top-5 ($98\%$ vs.\ $86\%$), while the conjunctive score has the stronger guarantee.
\end{tcolorbox}

\mypara{Sensitivity to Data Length.}
\textit{Can \methodname{} diagnose the root cause immediately upon failure detection, or must it wait for extended post-failure observations?} We vary the post-failure data from $10\%$ to $100\%$ (the first $p\%$ of observations) and compare with BARO (Figure~\ref{fig:robustness-data-lengths}), the baseline that $\text{\methodname}_\text{Marginal}$ reproduces (Section~\ref{subsec:ablation}). \methodname{} achieves its highest accuracy with early observations, reaching $81\%$ Top-1 at $10\%$ data, then decreasing to $68\%$ with full data. A likely reason is that the root cause deviates at fault onset, while propagated symptoms at affected components grow over time. Because each property is scored by its peak post-failure value, a longer window captures more propagated and noisy excursions. BARO degrades more, from $46\%$ Top-1 at $10\%$ data to $19\%$ at full data. As fault propagation continues, downstream components can accumulate increasingly severe symptoms, which may eventually exceed the root cause's anomaly (cf.\ Theorem~\ref{thm:counter}). Under Assumptions~\ref{ax:direction}--\ref{ax:isolation}, propagation does not raise the internal scores of affected components, which may partly explain why \methodname{} degrades more slowly than BARO, although the two methods also differ in scorer and in ranking components rather than properties.

\begin{figure}
\centering
\includegraphics[width=\columnwidth]{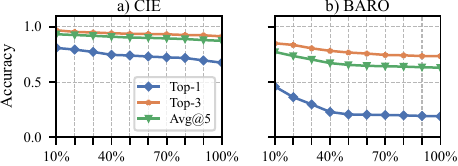}
\vspace{-5pt}
\caption{Accuracy vs.\ fraction of post-failure data.} \label{fig:robustness-data-lengths}
\vspace{-10pt}
\end{figure}

\begin{tcolorbox}[colormetabox]
\textbf{Finding 8:} \methodname{} is most accurate with early observations ($81\%$ Top-1 at $10\%$ data), and a single random post-failure sample per metric yields $65\%$ Top-1 (Appendix~\ref{app:single-sample}). With full data, \methodname{} ($68\%$) exceeds CIRCA ($56\%$) and BARO ($19\%$).
\end{tcolorbox}

\subsection{Ablation Study}\label{subsec:ablation}

For RQ4, we compare full \methodname{} with three variants (Table~\ref{tab:ablation}). $\text{\methodname}_\text{Marginal}$ ranks individual properties by their marginal IQR-based score, as BARO does. $\text{\methodname}_\text{Internal}$ and $\text{\methodname}_\text{External}$ rank components by their pooled internal or external IQR-based score.

\begin{table}[t]
\centering
\caption{Ablation study.} \label{tab:ablation}
\resizebox{\columnwidth}{!}{\begin{tabular}{lccccc}
\toprule
Method & Top-1 & Top-3 & Top-5 & Avg@5 & Time (ms) \\
\midrule
$\text{\methodname}_\text{Marginal}$ & 0.19 & 0.74 & 0.85 & 0.63 {\scriptsize (+)} & 10.6 \\
$\text{\methodname}_\text{Internal}$ & 0.15 & 0.65 & 0.77 & 0.56 {\scriptsize (+)} & 11.1 \\
$\text{\methodname}_\text{External}$ & \underline{0.54} & \underline{0.81} & \underline{0.89} & \underline{0.77} {\scriptsize (+)} & 11.9 \\
\textbf{\methodname{}} & \textbf{0.68} & \textbf{0.91} & \textbf{0.98} & \textbf{0.87} & 12.6 \\
\bottomrule
\end{tabular}
}\end{table}

$\text{\methodname}_\text{Marginal}$ reaches $19\%$ Top-1, matching BARO, which can rank first an unrelated service with a near-zero pre-failure IQR (Table~\ref{tab:case-study}).

$\text{\methodname}_\text{Internal}$ has the lowest Top-1 ($15\%$), possibly because a fault, such as a code-level bug, need not produce a resource spike. $\text{\methodname}_\text{External}$ reaches $54\%$ Top-1, consistent with downstream effects having larger external anomalies than the root cause (Theorem~\ref{thm:counter}).

\begin{tcolorbox}[colormetabox]
\textbf{Finding 9:} Ranking by individual properties or by a single property type is far less accurate than full \methodname{} ($15$--$54\%$ vs.\ $68\%$ Top-1). Because the variants also use the IQR-based scorer, these gaps confound the scorer with the internal--external combination.
\end{tcolorbox}
 
\section{Conclusion}\label{sec:conclusion}
We asked how to find root causes when the dependency graph is unknown, and showed that ranking by marginal anomaly scores can fail when downstream effects accumulate. \methodname{} combines internal and external evidence instead. Under the Component-Property Model, its conjunctive score ranks the root cause first, regardless of external amplification, whenever both root-cause scores exceed a threshold that bounds the internal scores of all other components. Its additive score does so under bounded amplification and a larger threshold on the summed score. Empirically, \methodname{} achieves 68\% Top-1 accuracy on 735 failure cases, a 21\% relative improvement over the best of twelve baselines, with a median runtime of 8.2\,ms per case.

\mypara{Limitations.} \methodname{} requires a one-time, approximate internal/external labeling of each metric type (Section~\ref{sec:system-model}). Retries, queue buildup, and shared-resource contention violate Assumptions~\ref{ax:direction}--\ref{ax:isolation} when they raise the internal metrics of non-root components (Appendices~\ref{app:cpm-scope} and~\ref{app:contention}). The guarantees also exclude network faults, about a third of the cases. Our evaluation uses nine datasets of injected faults in three microservice systems, and we selected the default configuration on the same data. \methodname{} also assumes a known failure time, and its peak-based scorer targets upward excursions.
 
\subsection*{AI Use Statement}
We used generative AI tools (Claude by Anthropic) to help refine the theoretical framework and review the written proofs. We also used them to assist in writing parts of experiment code and to format the manuscript and references. We did not use them to generate or clean datasets, propose hypotheses, translate, or analyze qualitative data.
The authors checked all AI feedback on the proofs, verified AI-assisted code by running it on the benchmark, and take full responsibility for the content of this paper.
 
\bibliographystyle{unsrtnat}

\clearpage
\appendix
\thispagestyle{empty}
\onecolumn
\aistatstitle{Graph-Free Root Cause Analysis\\ Supplementary Materials}

\section{Proofs}\label{app:proofs}

We prove the results of Section~\ref{sec:cie} and Appendix~\ref{subsec:it-saturation-theory}. Throughout, $C_r$ is the root cause, $C_a$ with $a \neq r$ ranges over all non-root components, and $s^I_i = S^I(C_i)$, $s^E_i = S^E(C_i)$. A function is \emph{weakly monotone} if it is nondecreasing in each argument. Strict monotonicity (Definition~\ref{def:monotonicity}) applies only to property-level scorers.

\subsection{Anomaly Scoring}

\begin{proof}[Proof of Lemma~\ref{lem:mono-inject} (Monotonicity implies Injectivity)]
Suppose $|x-c| \neq |y-c|$ and, without loss of generality, $|x-c| > |y-c|$. By monotonicity, $S(x) > S(y)$, hence $S(x) \neq S(y)$.

For the converse, let $S(x) = 1/(1+|x-c|)$, which maps the full domain into $(0,1] \subset \mathbb{R}_{\geq 0}$. It depends on $x$ only through $|x-c|$ and is strictly decreasing in $|x-c|$, so distinct deviations receive distinct scores and $S$ is injective in the sense of Definition~\ref{def:injectivity}. However, $|x-c| > |y-c|$ implies $S(x) < S(y)$, so $S$ is not monotonic.
\end{proof}

\begin{proof}[Proof of Theorem~\ref{thm:deviation-conditions} (Conditions Satisfied by Deviation-Based Scores)]
Given the reference data $\theta$, the location $c = c(\theta)$ and the scale $s = s(\theta) > 0$ are fixed, and $S(x) = |x - c|/s$.

\textbf{Monotonicity.} If $|x - c| > |y - c|$, dividing both sides by $s > 0$ preserves the inequality, so $S(x) = |x-c|/s > |y-c|/s = S(y)$.

\textbf{Injectivity.} Follows from monotonicity by Lemma~\ref{lem:mono-inject}.

The property-level score $S(P) = S(\max_{t \in \mathcal{T}_{\text{post}}} x_t)$ applies the same $S$ to the peak post-failure value, so it is monotonic and injective as a function of that peak value. It is not monotonic in the deviation of the whole trajectory. For example, a trajectory that drops far below $c$ but also contains a value near $c$ receives a small score. The guarantees of this appendix concern the peak statistic.
\end{proof}

\begin{proof}[Proof of Theorem~\ref{thm:rca-necessity} (Injectivity is Necessary)]
Since $S$ is not injective, there exist $x, y$ with $|x - c| \neq |y - c|$ and $S(x) = S(y) =: \sigma$. Without loss of generality, $|x-c| > |y-c|$. Consider two components $C_1, C_2$, each with one property scored by $S$ against the same reference, and two failure scenarios.
\begin{itemize}[nolistsep]
    \item Scenario A. $C_1$ observes $x$, $C_2$ observes $y$, and the root cause is $C_1$.
    \item Scenario B. $C_1$ observes $y$, $C_2$ observes $x$, and the root cause is $C_2$.
\end{itemize}
In each scenario, the root cause has the strictly largest deviation, so the deviations identify it. Both scenarios, however, produce the same score vector $(\sigma, \sigma)$. A ranking rule that depends on the data only through $S$ receives identical inputs in the two scenarios and therefore returns the same ranking (or, if randomized, the same distribution over rankings). A deterministic rule ranks the same component first in both scenarios and fails in one of them. A randomized rule that ranks $C_1$ first with probability $p$ succeeds with probability $p$ in Scenario A and $1-p$ in Scenario B, so it fails with probability at least $1/2$ in one of them. Adding components with zero deviation does not change the argument. The construction also fits the Component-Property Model. Give each component an external property that observes the same value $z$ in both scenarios, and let the internal properties carry $x$ and $y$ as above. The score vectors stay identical across the two scenarios, so the conclusion holds.
\end{proof}

\subsection{Combiners}

\begin{proof}[Proof of Proposition~\ref{prop:add-valid} (Additive Score)]
Write $M_{\text{add}}(s^I, s^E) = g(s^I + s^E)$ with $g(t) = t - \log(1+t)$. Since $g'(t) = t/(1+t) \geq 0$, with equality only at $t = 0$, $g$ is continuous and strictly increasing on $\mathbb{R}_{\geq 0}$, and $g(t) \geq 0$ because $\log(1+t) \leq t$.

\textbf{Weak monotonicity.} $\partial M_{\text{add}}/\partial s^I = \partial M_{\text{add}}/\partial s^E = g'(s^I + s^E) \geq 0$, with strict inequality whenever $s^I + s^E > 0$.

\textbf{No internal boundedness on $\mathbb{R}_{\geq 0}^2$.} For any candidate $f$, $M_{\text{add}}(0, s^E) = g(s^E) \to \infty$ as $s^E \to \infty$, so no finite value $f(0)$ bounds $M_{\text{add}}(0, \cdot)$.

\textbf{Internal boundedness on $D_{\gamma,\delta}$.} If $s^E \leq \gamma s^I + \delta$, then $s^I + s^E \leq (1+\gamma)s^I + \delta$, and since $g$ is nondecreasing,
\[
M_{\text{add}}(s^I, s^E) = g(s^I + s^E) \leq g\bigl((1+\gamma)s^I + \delta\bigr) = f(s^I).
\]
The function $f$ maps $\mathbb{R}_{\geq 0}$ into $\mathbb{R}_{\geq 0}$ and is nondecreasing as a composition of nondecreasing maps.

\textbf{Root-cause condition.} Since $g$ is strictly increasing, $M_{\text{add}}(s^I, s^E) = g(s^I + s^E) > g((1+\gamma)\epsilon + \delta) = f(\epsilon)$ if and only if $s^I + s^E > (1+\gamma)\epsilon + \delta$.
\end{proof}

\begin{remark}[The additive score needs a quantified separation]\label{rem:add-separation}
$M_{\text{add}}$ does not satisfy threshold separation in general. When $s^I > \epsilon$ and $s^E > \epsilon$, the total $s^I + s^E$ ranges over $(2\epsilon, \infty)$, so by Proposition~\ref{prop:add-valid} threshold separation at $\epsilon$ holds if and only if $2\epsilon \geq (1+\gamma)\epsilon + \delta$, i.e., $(1-\gamma)\epsilon \geq \delta$. It fails, for instance, whenever $\gamma > 1$ and $\epsilon > 0$. For example, with $\epsilon = 1$, $\gamma = 10$, and $\delta = 0$, a root cause with $(s^I_r, s^E_r) = (1.1, 1.1)$ scores $M_{\text{add}} = 2.2 - \log 3.2 \approx 1.04$, whereas a non-root component with $(1, 10) \in D_{\gamma,\delta}$ scores $11 - \log 12 \approx 8.52$. Part~2 of Theorem~\ref{thm:cie-correct} therefore requires the root cause's total evidence to exceed the largest total $(1+\gamma)\epsilon + \delta$ that a non-root component can reach under bounded amplification. The linear bounding function $(1+\gamma)s + \delta \geq f(s)$ is also valid, since $g(t) \leq t$, but it leads to the stronger requirement $s^I_r + s^E_r > g^{-1}((1+\gamma)\epsilon + \delta)$.
\end{remark}

\begin{proof}[Proof of Proposition~\ref{prop:conj-valid} (Conjunctive Score)]
\textbf{Weak monotonicity.} If $s^I \leq s'^I$, then $\min(s^I, s^E) \leq \min(s'^I, s^E)$, and symmetrically in $s^E$. The minimum is not strictly increasing, e.g., $\min(1,2) = \min(1,3)$.

\textbf{Internal boundedness.} $\min(s^I, s^E) \leq s^I = f(s^I)$, where the identity $f(s) = s$ is nondecreasing.

\textbf{Threshold separation.} If $s^I > \epsilon$ and $s^E > \epsilon$, then $\min(s^I, s^E) > \epsilon = f(\epsilon)$, for every $\epsilon \geq 0$.
\end{proof}

\begin{remark}[Internal boundedness alone is not sufficient]\label{rem:bound-not-sufficient}
Let $M(s^I, s^E) = s^I + 1 - e^{-s^E}$. It is strictly increasing in both arguments and internally bounded with $f(s) = s + 1$. With $\epsilon = 1$, a root cause with $(s^I_r, s^E_r) = (1.1, 1.1)$ and a non-root component with $(s^I_a, s^E_a) = (1, 100)$ satisfy $s^I_a \leq \epsilon < \min(s^I_r, s^E_r)$, yet $M(C_r) \approx 1.767 < M(C_a) \approx 2.000$. The combiner violates threshold separation at $\epsilon = 1$, since $M(s^I, s^E)$ approaches $2 - e^{-1} \approx 1.63 < f(1) = 2$ as $(s^I, s^E) \to (1, 1)$ from above. Threshold separation (Definition~\ref{def:threshold-sep}) excludes such combiners.
\end{remark}

Internal boundedness is not vacuous either. For $M = \max(s^I, s^E)$ and for the external-only score $M = s^E$, $M(0, s^E) = s^E$ is unbounded in $s^E$, so no $f$ with $M(s^I, s^E) \leq f(s^I)$ exists.

\subsection{Separation and Ranking Guarantees}

We make the probabilistic setting explicit. For a component $C_a$, let $X_a$ denote the trajectories of its internal properties over the pre-failure window and the post-failure window $\mathcal{T}_{\text{post}}$. The pooled internal score $S^I(C_a) = \phi(\{S(P) : P \in \mathbf{I}_a\})$ is a fixed function $\Psi_a(X_a)$, since the reference parameters $c$ and $s$ of each property are computed from its pre-failure values and the peak from its post-failure values. The \emph{null distribution} $F_a^0$ is the distribution of $\Psi_a(X_a)$ in fault-free operation, with the same window lengths. For $\alpha_a \in (0,1)$, we set $\epsilon_a = \inf\{t : F_a^0(t) \geq 1 - \alpha_a\}$, the $(1-\alpha_a)$-quantile, so that $F_a^0(\epsilon_a) \geq 1 - \alpha_a$ by right-continuity of $F_a^0$. \emph{Stationarity} means that, absent an intervention on $C_a$, the internal properties of $C_a$ follow the same distribution before and after the fault time, namely their fault-free distribution. The null distribution then does not depend on where the windows are placed and can be estimated from historical fault-free windows. Random quantities are the observed trajectories. The fault, the root cause index $r$, the levels $\alpha_a$, and the thresholds $\epsilon_a$ are fixed.

\begin{proof}[Proof of Lemma~\ref{lem:internal-separation}]
\textbf{Part (1).} If $\phi(x) \geq \max_i x_i$, then $S^I(C_r) = \phi(\{S(P) : P \in \mathbf{I}_r\}) \geq S(I_r^j)$. Both $\max$ and $\mathrm{sum}$ satisfy this for non-negative arguments. $\mathrm{mean}$ does not, and it yields only the weaker bound $S^I(C_r) \geq S(I_r^j)/|\mathbf{I}_r|$ because all scores are non-negative.

\textbf{Part (2).} Fix $a \neq r$. By Assumption~\ref{ax:isolation}, no property of another component has a direct causal edge into $\mathbf{I}_a$. By Assumption~\ref{ax:direction}, no external property of $C_a$ has an edge into $\mathbf{I}_a$. Every causal parent of a property in $\mathbf{I}_a$ therefore lies in $\mathbf{I}_a$, and so does every causal ancestor. The fault is an intervention on $I_r^j \notin \mathbf{I}_a$, which is therefore not an ancestor of $\mathbf{I}_a$. By the third clause of Assumption~\ref{asm:sufficiency}, the fault is also not the effect of an unmeasured common cause that acts on $\mathbf{I}_a$ as well (such as contention on a shared host, see Appendix~\ref{app:contention}). Hence the fault is not an intervention on $C_a$, directly or through a parent, and by stationarity $X_a$ follows its fault-free distribution. Consequently $S^I(C_a) = \Psi_a(X_a) \sim F_a^0$ and
\[
P\bigl(S^I(C_a) \leq \epsilon_a\bigr) = F_a^0(\epsilon_a) \geq 1 - \alpha_a.
\]
The argument uses only the distribution of the pooled statistic, so it holds for any pooling function $\phi$. A threshold calibrated per property would not suffice. If $\epsilon$ were the $0.95$-quantile of a single property's score and $C_a$ had 20 independent internal properties, max pooling would give $P(S^I(C_a) \leq \epsilon) = 0.95^{20} \approx 0.36$. Calibrating the pooled statistic per component removes this dependence on the number of properties.

\textbf{Part (3).} By the union bound and Part (2),
\[
P\bigl(\exists\, a \neq r : S^I(C_a) > \epsilon_a\bigr) \leq \sum_{a \neq r} P\bigl(S^I(C_a) > \epsilon_a\bigr) \leq \sum_{a \neq r} \alpha_a.
\]
No independence between components is needed.
\end{proof}

\begin{proof}[Proof of Theorem~\ref{thm:cie-correct} (\methodname{} Ranking Correctness)]
Fix $\epsilon \geq 0$ and assume $s^I_a \leq \epsilon$ for every $a \neq r$.

\textbf{Part (1).} Let $M$ be internally bounded with bounding function $f$ and satisfy threshold separation at $\epsilon$. For every $a \neq r$, internal boundedness and the monotonicity of $f$ give
\[
M(C_a) = M(s^I_a, s^E_a) \leq f(s^I_a) \leq f(\epsilon).
\]
Since $s^I_r > \epsilon$ and $s^E_r > \epsilon$, threshold separation gives $M(C_r) > f(\epsilon)$. Hence $M(C_r) > f(\epsilon) \geq M(C_a)$ for every $a \neq r$. For $M_{\text{conj}}$, Proposition~\ref{prop:conj-valid} provides $f(s) = s$ on all of $\mathbb{R}_{\geq 0}^2$ and threshold separation at every $\epsilon$, so no condition on $s^E_a$ is needed.

\textbf{Part (2).} For every $a \neq r$, $(s^I_a, s^E_a) \in D_{\gamma,\delta}$, so Proposition~\ref{prop:add-valid} gives
\[
M_{\text{add}}(C_a) \leq g\bigl((1+\gamma)s^I_a + \delta\bigr) \leq g\bigl((1+\gamma)\epsilon + \delta\bigr) < g(s^I_r + s^E_r) = M_{\text{add}}(C_r),
\]
where the last inequality uses the root-cause condition and the strict monotonicity of $g$. Part~(2) does not require $s^I_r > \epsilon$ individually.

In both parts the inequality is strict, so $C_r$ is ranked first without ties. The argument does not use the weak monotonicity of $M$. We require it of every combiner so that stronger evidence never lowers a score.
\end{proof}

\begin{proof}[Proof of Corollary~\ref{cor:prob-first} (Probabilistic Guarantee)]
Let $\epsilon^* = \max_{a \neq r} \epsilon_a$ and define the events
\[
A = \{S^I(C_a) \leq \epsilon_a \ \forall\, a \neq r\}, \qquad R = \{\min(s^I_r, s^E_r) > \epsilon^*\}.
\]
On $A$, $s^I_a \leq \epsilon_a \leq \epsilon^*$ for every $a \neq r$. On $A \cap R$, Part~(1) of Theorem~\ref{thm:cie-correct} with $\epsilon = \epsilon^*$ gives $M_{\text{conj}}(C_r) > M_{\text{conj}}(C_a)$ for every $a \neq r$. Hence, by Part~(3) of Lemma~\ref{lem:internal-separation},
\[
P(C_r \text{ ranked first by } M_{\text{conj}}) \geq P(A \cap R) \geq 1 - \sum_{a \neq r} \alpha_a - P(R^c).
\]
The event $R$ is the quantitative form of the second clause of Assumption~\ref{asm:sufficiency}. The bound is unconditional and needs no independence between $R$ and $A$. If manifestation holds surely, $P(R^c) = 0$ and the bound becomes $1 - \sum_{a \neq r}\alpha_a$. For $M_{\text{add}}$, replace $R$ by the event $G \cap R'$, where $G = \{s^E_a \leq \gamma s^I_a + \delta \ \forall\, a \neq r\}$ and $R' = \{s^I_r + s^E_r > (1+\gamma)\epsilon^* + \delta\}$, and apply Part~(2) of Theorem~\ref{thm:cie-correct}. This gives $P(C_r \text{ ranked first by } M_{\text{add}}) \geq 1 - \sum_{a \neq r}\alpha_a - P(G^c) - P(R'^c)$.
\end{proof}

The union bound is conservative, since it holds under arbitrary dependence between components. With the Bonferroni choice $\alpha_a = \alpha/(n-1)$, the bound becomes $1 - \alpha$.

\begin{corollary}[Ranking Margin]\label{cor:ranking-margin}
Under the conditions of Part~(1) of Theorem~\ref{thm:cie-correct}, the conjunctive score $M_{\text{conj}} = \min(S^I, S^E)$ satisfies
\[
M_{\text{conj}}(C_r) - \max_{a \neq r} M_{\text{conj}}(C_a) \geq \min(s^I_r, s^E_r) - \epsilon > 0.
\]
The margin grows with the severity of the root cause anomaly.
\end{corollary}

\begin{proof}
By Proposition~\ref{prop:conj-valid}, $M_{\text{conj}}(C_a) \leq f(s^I_a) = s^I_a \leq \epsilon$ for every $a \neq r$. For the root cause, $M_{\text{conj}}(C_r) = \min(s^I_r, s^E_r) > \epsilon$ because both scores exceed $\epsilon$. Subtracting the two bounds gives the margin.
\end{proof}

\subsection{IT-Score Saturation and Marginal Score Ordering}

\begin{proof}[Proof of Theorem~\ref{thm:it-saturation} (IT-Score Saturation)]
Let $N(x) = |\{i \in [k] : \tau(x_i) \geq \tau(x)\}| \in \{0, 1, \ldots, k\}$.
\begin{enumerate}[nolistsep]
    \item If $N(x) \geq 1$, then $N(x)/k \in [1/k, 1]$, so $\hat{S}_{\text{IT}}(x) = -\log(N(x)/k) \in [0, \log k]$.
    \item If $\tau(x) > \max_{i \in [k]} \tau(x_i)$, no reference observation satisfies $\tau(x_i) \geq \tau(x)$, so $N(x) = 0$. The smoothed estimator then returns $-\log\frac{0.5}{k + 0.5} = \log(2k+1)$, a constant that does not depend on $\tau(x)$. Since $2k + 1 > k$, this value exceeds every score in Part~1 and is the largest attainable score.
    \item Any two observations that exceed the reference maximum have $N = 0$ and receive the same score $\log(2k+1)$, even if their $\tau$-values, and hence their deviations, differ.
\end{enumerate}
If Laplace smoothing is applied to every count, $\hat{S}_{\text{IT}}(x) = -\log\frac{N(x) + 0.5}{k + 0.5}$, then $N(x) \geq 1$ gives at most $\log\frac{2k+1}{3} < \log(2k+1)$ and $N(x) = 0$ gives the same constant $\log(2k+1)$, so all three conclusions are unchanged.
\end{proof}

\begin{proof}[Proof of Theorem~\ref{thm:counter} (Counterexample to Marginal Score Ordering)]
Consider a caller $C_2$ that, for each request, makes $k \geq 2$ sequential calls to a callee $C_1$, the root cause. Let the caller's latency be the sum of the callee latencies of these calls, $L_2 = \sum_{j=1}^{k} L_1^{(j)}$, where the caller's own constant processing time is omitted. Each component has a single external property, its latency, so $S^E$ is the latency score under $\max$, $\mathrm{sum}$, or $\mathrm{mean}$ pooling. Before the fault, the per-call latencies are independent with mean $L_0$ and standard deviation $\sigma_0 > 0$, so $L_2$ has mean $k L_0$ and standard deviation $\sqrt{k}\,\sigma_0$. The fault is an intervention on an internal property of $C_1$ (e.g., CPU contention) that adds a deterministic delay $\Delta > 0$ to every call. The caller's latency changes only through the callee's latency, an external-to-external influence, and the caller's internal properties are untouched, as Assumptions~\ref{ax:direction}--\ref{ax:isolation} require.

\textbf{Baseline-normalized shift.} After the fault, the callee's mean latency shifts by $\Delta$, which is $\Delta/\sigma_0$ in units of its own baseline. The caller's mean latency shifts by the raw cumulative delay $k\Delta$, which is $k\Delta/(\sqrt{k}\,\sigma_0) = \sqrt{k}\,\Delta/\sigma_0$ in units of its own baseline. After normalization, the caller's deviation is $\sqrt{k}$ times the callee's, not $k$ times.

\textbf{Peak-based z-score.} We use the pre-failure mean and standard deviation as $c$ and $s$ (the large-window limit of the estimates). Let $Z_1$ and $Z_2$ be the standardized peaks of the fault-free fluctuations of $L_1$ and $L_2$ over the post-failure window, i.e., the signed maxima $\max_t (x_t - \mu)/\sigma$ of the fault-free trajectories, where $\mu$ is the baseline mean and $\sigma$ is $\sigma_0$ for $L_1$ and $\sqrt{k}\,\sigma_0$ for $L_2$. Because the shift is deterministic, the post-failure peak equals the fault-free peak plus the shift, so
\[
S^E(C_1) = \bigl|Z_1 + \Delta/\sigma_0\bigr|, \qquad S^E(C_2) = \bigl|Z_2 + \sqrt{k}\,\Delta/\sigma_0\bigr|.
\]
When both arguments are positive,
\[
S^E(C_2) - S^E(C_1) = (\sqrt{k} - 1)\,\Delta/\sigma_0 - (Z_1 - Z_2),
\]
which is positive whenever $\Delta/\sigma_0 > (Z_1 - Z_2)/(\sqrt{k} - 1)$. If the per-call fluctuations are bounded by $B\sigma_0$ in absolute value, then $|Z_1| \leq B$ and $|Z_2| \leq \sqrt{k}\,B$, so both arguments are positive and $S^E(C_2) > S^E(C_1)$ surely once $\Delta/\sigma_0 > B(1+\sqrt{k})/(\sqrt{k}-1)$. For unbounded fluctuations, such as Gaussian ones, $Z_1$ and $Z_2$ are finite random variables for a fixed window length, so $S^E(C_2) > S^E(C_1)$ with probability tending to one as $\Delta/\sigma_0$ grows. In either case, the root cause does not have the highest marginal external score.

\textbf{$M_{\text{conj}}$ is unaffected.} The caller's internal properties stay at nominal levels, whereas the callee's internal and external properties manifest the fault. On the separation event, Part~(1) of Theorem~\ref{thm:cie-correct} gives $M_{\text{conj}}(C_1) > M_{\text{conj}}(C_2)$ for any value of $k$.

\textbf{$M_{\text{add}}$ needs a larger internal anomaly.} Since $M_{\text{add}}$ ranks by $s^I + s^E$ (Proposition~\ref{prop:add-valid}), it ranks $C_1$ above $C_2$ if and only if $S^I(C_1) - S^I(C_2) > S^E(C_2) - S^E(C_1) = (\sqrt{k}-1)\,\Delta/\sigma_0 - (Z_1 - Z_2)$ when both arguments above are positive. The required internal gap grows with $\sqrt{k}$, so for a fixed root-cause internal anomaly $M_{\text{add}}$ ranks the caller first once $k$ is large enough. This is the bounded-amplification condition of Part~(2) of Theorem~\ref{thm:cie-correct}.
\end{proof}

\begin{proof}[Proof of Corollary~\ref{cor:deviation-rca} (Deviation-Based Scores Preserve the Deviation Order)]
Given the reference, $S(x) = |x-c|/s$ is a strictly increasing function of $|x-c|$. Hence $S(x) > S(y)$ if and only if $|x-c| > |y-c|$, and $S(x) = S(y)$ if and only if $|x-c| = |y-c|$. Scores are therefore tied only between equal deviations, and the ranking by $S$ coincides with the ranking by deviation. This is not sufficient for identification. In the system of Theorem~\ref{thm:counter}, the caller's baseline-normalized external deviation exceeds the root cause's, so ranking by the external score alone selects the caller. A marginal ranking over all properties also selects the caller whenever the caller's external score exceeds the root cause's internal score as well. The internal--external decomposition of Theorem~\ref{thm:cie-correct} is what resolves such cases.
\end{proof}

\begin{proof}[Proof of Corollary~\ref{cor:it-fails} (IT Scores Fail for RCA Under Saturation)]
By Part~2 of Theorem~\ref{thm:it-saturation}, each of the $m$ saturated components has a property with score $\log(2k+1)$, the largest attainable value, so its max-pooled score equals $\log(2k+1)$. Every other component has only properties with $N \geq 1$, so its pooled score is at most $\log k < \log(2k+1)$. The $m$ saturated components therefore tie at the top of the ranking. If $C_r$ is among them, a tie-breaking rule that is exchangeable among tied components selects each of them first with probability $1/m$, so $C_r$ is ranked first with probability exactly $1/m$. If $C_r$ is not among them, it is not ranked first. In both cases, Top-1 succeeds with probability at most $1/m$.
\end{proof}
 \section{Related Work}\label{app:related}

\mypara{Causal Graph-Based RCA.}
A prevalent approach to RCA constructs causal or dependency graphs from observational data and traverses them to identify root causes. Microscope~\cite{lin2018microscope} constructs service causal graphs by analyzing metrics time series data. CloudRanger~\cite{wang2018cloudranger} uses Pearson correlation to build dependency graphs and employs second-order random walks for root cause localization. MicroRCA~\cite{wu2021microrca} constructs attributed graphs from metrics and traces. CausalRCA~\cite{xin2023causalrca} applies DAG-GNN~\cite{yu2019daggnn} to learn causal structures and uses PageRank~\cite{page1999pagerank} for ranking. More recent work has explored alternative graph construction strategies. \citet{ikram2022root} propose RCD, a hierarchical approach that treats failures as interventions to reduce the number of conditional independence tests required for root cause localization. While RCD avoids learning the full causal graph, it still performs causal discovery as a per-failure step and inherits the data and assumption costs of conditional independence testing (e.g., causal sufficiency and a sample size that grows with the conditioning set). \methodname{} avoids structure learning altogether. Its property classification is a one-time labeling task that depends on metric semantics rather than on the system topology, and root cause ranking proceeds directly in $8.2$\,ms (median) without per-failure discovery. CIRCA~\cite{li2022causal} formulates RCA as intervention recognition in a Causal Bayesian Network constructed from call graphs. Following the RCAEval implementation, we run CIRCA on a graph that the PC algorithm learns from each case rather than on the dependency graphs we collected. BRCD~\cite{lee2026brcd} learns a completed partially directed acyclic graph from pre-failure data and computes the posterior probability that each candidate is the intervened node. RUN~\cite{LinCWWP24} employs neural Granger causal discovery with contrastive learning to capture temporal dependencies that PC-based methods miss. AERCA~\cite{han2025root} integrates Granger causal discovery with anomaly detection for multivariate time series, attributing anomalies to abnormal exogenous inputs of the learned model. Both RUN and AERCA must learn a temporal causal model of the monitored system from historical data before diagnosis, whereas \methodname{} needs only the pre-failure window of the current incident as reference data. EventADL~\cite{pham2026eventadl} constructs an intervention graph from events data and performs time-aware random walk over the constructed graph to localize the root cause interventions from given detected anomalies~\cite{ferreira2025hypergraph}. These methods share a common pipeline. They construct a causal graph using causal discovery methods or correlation-based approaches, then traverse or score nodes to identify root causes. However, as highlighted by~\citet{pham2024root}, these graph-based methods are often inefficient and ineffective. As shown in Section~\ref{subsec:effectiveness}, PC-PageRank achieves only $9\%$ Top-1 accuracy on the RCAEval benchmark. The graph-learning baselines CIRCA, DoWhy-GCM, RCD, and BRCD reach $56\%$, $49\%$, $35\%$, and $29\%$, respectively, but need median runtimes of seconds to minutes per case. \methodname{} bypasses graph construction entirely, achieving $68\%$ Top-1 accuracy while running $146\times$ faster than PC-PageRank (median $8.2$\,ms versus $1.2$\,s per case).

\mypara{Anomaly Score-Based RCA.}
An alternative approach ranks root cause candidates directly by anomaly scores without constructing causal graphs. \citet{budhathoki2022causal} formalize RCA as quantifying each variable's contribution to an outlier score using Shapley values, requiring a known causal graph and structural causal model. The DoWhy-GCM implementation~\cite{dowhygcm2024} of this approach reaches $49\%$ Top-1 in our evaluation, with a graph that the PC algorithm learns from each case. \citet{orchardroot} study RCA of outliers with missing structural knowledge. They take a hypothesis-testing view of anomaly scoring based on Information-Theoretic (IT) scores, and propose traversal methods that use the causal graph as well as a score-ordering method that needs no structural knowledge but assumes that the root cause exhibits the highest anomaly score. \methodname{} shares their goal of avoiding a known graph but replaces the ordering assumption with the internal--external asymmetry, and it uses deviation-based rather than IT scores, which saturate under strong anomalies (Appendix~\ref{subsec:it-saturation}). \citet{li2025root} exploit permutation invariance in linear structural causal models (SCMs) to identify root causes via Cholesky decomposition. \citet{schkoda2026root} trace outliers in linear structural equation models with possible feedback cycles and, without knowing the graph, narrow the candidates to a short list containing the true root causes and their parents on a cycle. Both approaches rely on linear structural equations that also hold in the anomalous regime, whereas \methodname{} makes no parametric assumption on the propagation mechanism and instead relies on the Component-Property Model. BARO~\cite{pham2024baro} combines Bayesian online change point detection with IQR-based scoring for robust anomaly quantification. TORAI~\cite{pham2026torai} computes the anomaly severity of each component, then clusters components into normal and abnormal groups for further causal inference-based RCA. \citet{nagalapattirobust} propose In-Distribution Interventions (IDI), which identifies root causes by testing both anomaly and fix conditions using interventional rather than counterfactual estimates, improving robustness when anomalies fall outside training distributions. Like the Shapley-based approach, IDI fits a structural causal model over a given causal graph, which \methodname{} does not require. From a theoretical perspective, \citet{ikram2025root} establish a connection between RCA and Interactive Graph Search, proving that any algorithm relying solely on marginal invariance tests requires at least $\Omega(\log n)$ tests. A common assumption across the graph-free score-based methods is that the root cause exhibits the \textit{highest} anomaly score among all candidates. However, as we prove in Theorem~\ref{thm:counter}, this assumption fails when downstream effects accumulate. For example, when a callee is invoked multiple times, its small delay can amplify into a larger delay at the caller. \methodname{} addresses this by combining internal and external anomalies. Under the Component-Property Model, when the internal scores of all non-root components stay below a threshold that both scores of the root cause exceed, its conjunctive score ranks the root cause first under arbitrary external amplification, and its additive score does so under bounded amplification (Theorem~\ref{thm:cie-correct}).

\mypara{LLM-Based RCA.}
Recent work explores large language models (LLMs) for RCA, particularly in cloud environments. RCLAgent~\cite{zhang2025adaptive}, which we use as a baseline, employs multi-agent recursion-of-thought to iteratively localize root causes using topology graphs and call graph information. Similarly, Stratus~\cite{chen2025stratus} proposes a multi-agent system for autonomous reliability engineering in modern clouds. At the industry scale, RCAgent~\cite{wang2024rcagent} is a tool-augmented LLM agent deployed at Alibaba Cloud that autonomously gathers evidence and reasons about root causes. Microsoft has deployed multiple LLM-based solutions. RCACopilot~\cite{chen2024rcacopilot} uses GPT-4 with retrieval-augmented generation for automated incident diagnosis, while \citet{ahmed2023recommending} focus on recommending both root causes and mitigation steps for Azure incidents. The OpenRCA benchmark~\cite{xu2025openrca} provides a systematic evaluation framework with 335 failures and 68GB of telemetry data from enterprise systems, revealing that even with a specialized RCA-agent, Claude 3.5 solves only 11.34\% of failure cases. While these methods leverage rich contextual information from logs, traces, and documentation, they incur substantial latency. RCLAgent requires $79.3$ seconds (median) per failure case compared to \methodname{}'s $8.2$ milliseconds, about $9{,}700\times$ slower. Furthermore, their effectiveness depends on call graphs and heterogeneous data sources that may not always be available in production systems. \methodname{} is a lightweight alternative that uses only metrics data, with theoretical guarantees.
  \section{Baselines}\label{app:baselines}

We compare \methodname{} against twelve RCA methods, grouped by whether they require structural knowledge.

\mypara{Methods that do not require structural knowledge.}
\begin{itemize}[nolistsep,leftmargin=*]
\item \textbf{BARO}~\cite{pham2024baro} applies an IQR-based anomaly scorer to pre- and post-failure observational data and ranks the most anomalous property as the root cause.
\item \textbf{Score-Ordering}~\cite{orchardroot} uses an Information-Theoretic (IT) anomaly scorer and ranks the component with the highest anomaly score as the root cause.
\item \textbf{Cholesky}~\cite{li2025root} uses the Cholesky decomposition of the covariance matrix of observational data, exploiting permutation invariance to identify the root cause.
\item \textbf{PC-PageRank}~\cite{pham2025rcaeval} constructs a causal graph from observational data using the PC algorithm~\cite{spirtes1991pc}. It then uses PageRank~\cite{page1999pagerank} to find the root cause, assuming the root cause affects other components during the failure time.
\item \textbf{RCD}~\cite{ikram2022root} adds a binary indicator that separates pre- and post-failure samples, discretizes each metric into five bins, and runs a localized PC search with chi-square tests on random subsets of five metrics to find the metrics adjacent to the indicator. It ranks these metrics by the strength of their dependence on the indicator. We keep all adjacent metrics, rank each component by its best-ranked metric, and fix the random seed to 0.
\item \textbf{BRCD}~\cite{lee2026brcd} learns a completed partially directed acyclic graph over the metrics from the pre-failure data with Best Order Score Search (BOSS)~\cite{andrews2023boss} and treats the failure as an intervention on one candidate metric. It ranks the candidates by their posterior under a linear-Gaussian model of the pre- and post-failure data, summing over the orientations compatible with the intervention. We run the authors' released code with a single learned graph (no bootstrap) on metrics standardized by their pre-failure statistics, and score each component by the maximum posterior of its metrics. On RE1OB, this configuration reaches $38\%$ Top-1, below the $62\%$ that~\citet{lee2026brcd} report for BRCD with a BOSS-learned graph. We could not check that setting against their released RCAEval code, which builds the graph from a service map instead of learning it.
\item \textbf{CIRCA}~\cite{li2022causal} formulates RCA as intervention recognition. It regresses each metric on its parents in a causal graph, scores the metric by its largest standardized post-failure residual, and ranks each component by the maximum score of its metrics. As published, CIRCA builds the graph from call relations. Following the RCAEval implementation, we instead learn the graph from all samples of each case with the PC algorithm~\cite{spirtes1991pc} (Fisher-$z$ test, $\alpha=0.05$) and omit the descendant-adjustment step of the full CIRCA model.
\item \textbf{DoWhy-GCM}~\cite{dowhygcm2024} implements the Shapley-based attribution of~\citet{budhathoki2022causal}. It fits a graphical causal model with automatically selected mechanisms on the pre-failure data, takes as target the component with the largest post-failure z-score, and ranks components by the mean Shapley contribution of their noise terms to the target's anomaly over the post-failure samples. Components other than the target and its ancestors score zero. The model has one node per component (the mean of its standardized metrics), and its DAG is learned from the pre-failure data by the PC algorithm (Fisher-$z$ test, $\alpha=0.05$, conditioning sets of at most three variables). We fit the model on at most 4{,}000 pre-failure samples and attribute at most 500 evenly spaced post-failure samples. Counterfactual uses the same Shapley attribution with several differences. It uses the collected dependency graph, represents each component by one arbitrarily chosen metric, takes the target from the first post-failure sample, attributes only that sample, and attributes the mean deviation of the target rather than its IT anomaly score, ranking by the absolute attribution.
\end{itemize}

\mypara{Methods that require structural knowledge.}
\begin{itemize}[nolistsep,leftmargin=*]
\item \textbf{Counterfactual}~\cite{hardt2024petshop} represents each component of the dependency graph by one of its metrics, fits a structural causal model on the pre-failure data, takes as target the component with the largest z-score at the first post-failure sample, and ranks components by the Shapley contribution of each to the target's anomaly~\cite{budhathoki2022causal}.
\item \textbf{RCLAgent}~\cite{zhang2025adaptive} uses LLMs to recursively analyze anomalous components along a call graph from traces.
\item \textbf{Simple-Traversal}~\cite{orchardroot} identifies nodes that are anomalous, have no anomalous parents, and are linked to the target node via a path in the dependency graph.
\item \textbf{Smooth-Traversal}~\cite{orchardroot} extends Simple-Traversal with IT anomaly scoring.
\end{itemize}

For BARO, Score-Ordering, Cholesky, PC-PageRank, and the four methods that require structural knowledge, we set hyperparameters as recommended in their papers. For those of the eight with configurable thresholds, we run multiple settings and report the best performance, following established practice. RCLAgent calls an LLM through a commercial API, and the supplementary material records the provider and model in its configuration file. All methods use the benchmark fault-injection time to split each case into pre- and post-failure data.

\mypara{Runs of RCD, BRCD, CIRCA, and DoWhy-GCM.} We ran each of these four baselines in a single configuration, using a separate harness that reads the same RCAEval metric files and applies the same split. Each method applies its own preprocessing, such as filling gaps and dropping metrics that are constant before the failure. The runs used CPU nodes of a shared cluster with 4 cores and 16\,GB of memory per case (2 cores and 8\,GB for CIRCA). The time limit per case was one hour for RCD and BRCD, three hours for CIRCA, and four hours for DoWhy-GCM. One RCD case was rerun with a four-hour limit, and two CIRCA cases that ran out of memory were rerun with 32\,GB. BRCD exceeded its limit on 13 RE2TT cases, and DoWhy-GCM on 158 cases, 131 of them among the 245 Train Ticket cases. We credit each such case with the exact expected Top-$k$ of a uniformly random ranking of its components. On the 577 cases that DoWhy-GCM completes, its Top-1 accuracy is close to that of \methodname{} ($65\%$ vs.\ $67\%$ of these cases, exact McNemar test $p=0.06$), but its Top-3 accuracy is much lower ($71\%$ vs.\ $90\%$).

RCD and DoWhy-GCM give a nonzero score to only a few components per case (median 7 and 4, respectively, the latter on the cases DoWhy-GCM completes). These are the neighbors of RCD's failure indicator and, for DoWhy-GCM, the target and its ancestors. A root cause outside this set ties with the other components at zero, consistent with both methods' small gains from Top-1 to Top-3 (Table~\ref{tab:rca_rcaeval}). As with all methods, the evaluation breaks ties in the order in which each implementation lists the components. Breaking them uniformly at random instead changes the overall Top-3 and Avg@5 of RCD and DoWhy-GCM by less than $0.01$. RCD and DoWhy-GCM use a fixed seed (0). BRCD is not seeded, and two further runs give overall Top-1 values of $0.30$ and $0.29$ (the reported run gives $0.29$).

Two RE1OB cases (\service{currencyservice} LOSS~1 and \service{productcatalogservice} CPU~3) were updated in our copy of the benchmark data after \methodname{} and the other eight baselines were run, so only these four baselines use the updated files. On the original files, \methodname{} gives every component a score of zero in both cases. Under the list-order rule, this counts as a Top-1 hit in the first case, which raises the RE1OB LOSS Top-1 of \methodname{} by $0.04$ and its overall Top-1 by less than $0.001$. BARO returns no ranking for one of the two cases, and Counterfactual, Score-Ordering, Simple-Traversal, and Smooth-Traversal return none for either. PC-PageRank returns none for two Train Ticket cases. These methods are therefore evaluated on 734 or 733 cases.

\mypara{Property and Component Rankings.} BARO, Cholesky, and PC-PageRank rank properties (metrics), while the other methods, including \methodname{} and Counterfactual, rank components. RCD, BRCD, and CIRCA score metrics and rank each component by its best metric. For all methods, a ranked entry counts as a hit if its name contains the name of the root cause component, which also credits entries such as \service{carts-db} for \service{carts}. For a property ranking, Top-$k$ counts a hit if any of the top $k$ ranked properties belongs to the root cause component. Properties of the same component are not merged, so a component can occupy several of the top $k$ positions.

\mypara{Retrieving Structural Knowledge.} The RCAEval benchmark does not provide the structural dependency graph. To run the baselines that require this graph, a DevOps engineer deployed the systems used in RCAEval and instrumented them to capture the interactions between components (i.e., services). While this instrumentation captures the structural knowledge, it adds significant latency overhead. The collected graphs are included in the supplementary material. CIRCA and DoWhy-GCM, which take a graph as input, instead run on graphs that the PC algorithm learns from each case, as described above.

\mypara{Trace Availability for RCLAgent.} RCLAgent builds its topology graph from traces. Traces are unavailable for the RE1 datasets, RE2SS, and RE3SS, where RCLAgent produces no ranking, and RCLAgent returns no result for one further case. We therefore report its accuracy and running time on the remaining 239 cases (RE2OB, RE2TT, RE3OB, and RE3TT).

\mypara{Related Methods Not Evaluated.} RUN~\cite{LinCWWP24} and AERCA~\cite{han2025root} learn a neural temporal causal model of the monitored system before diagnosis, and we leave their evaluation to future work. IDI~\cite{nagalapattirobust} fits a structural causal model over a given causal graph, as does the Shapley-based attribution family that Counterfactual and DoWhy-GCM represent. The method of~\citet{schkoda2026root} assumes linear structural equations that also hold in the anomalous regime, as Cholesky does, and returns a short candidate list rather than a ranking. We run each baseline without the internal/external labels. RCD, BRCD, CIRCA, and DoWhy-GCM deviate from their published designs as described above, most notably in the learned graphs of CIRCA and DoWhy-GCM.
 \section{Additional Results}

\subsection{Effectiveness of \methodname{}} \label{appendix:rca-effectiveness}

This section presents detailed per-dataset results complementing the aggregate analysis in Section~\ref{subsec:effectiveness}. The following tables report performance across all fault types for each of the nine datasets in the RCAEval benchmark. The AVERAGE columns average over the fault types of each dataset. RCLAgent requires traces and appears only in the tables of datasets that provide them. Bold marks the best and underline the second-best value in each column. The p-value column is defined as in Table~\ref{tab:rca_rcaeval}. For BRCD and DoWhy-GCM, we credit cases over the time limit with the expected accuracy of a random ranking (Appendix~\ref{app:baselines}), so their accuracies need not be multiples of $1/n$ for $n$ cases. Each fault type has 15 to 25 cases in the RE1 and RE2 datasets but only 3 to 10 in the RE3 datasets, so per-fault-type accuracies on RE3 (e.g., $100\%$ Top-1) carry large uncertainty. The resource faults are CPU, memory (MEM), disk (DISK), and socket (SOCKET) stress, the network faults are network delay (DELAY) and packet loss (LOSS), and F1--F5 are the five code-level fault types of the RE3 datasets~\cite{pham2025rcaeval}. The overall results average each label F$j$ across the systems in which it occurs.

On per-case Top-1 over all 735 cases, an exact McNemar test also favors \methodname{} over the best baseline, CIRCA (\methodname{} alone ranks the root cause first in 157 cases and CIRCA alone in 55, $p=1.5\times10^{-12}$).

\subsubsection{Online Boutique (Unimodal)}

On RE1OB (Online Boutique, Table~\ref{tab:rca_re1ob}), \methodname{} achieves $61\%$ Top-1 and $92\%$ Top-3 accuracy. Performance is strongest on MEM ($84\%$ Top-1) and DELAY ($76\%$ Top-1) faults. LOSS faults are the most challenging ($28\%$ Top-1), as packet loss may not immediately manifest in internal resource metrics. On individual fault types, several baselines rank the root cause first more often than \methodname{}, e.g., RCD on CPU ($88\%$ vs.\ $52\%$) and CIRCA and RCD on MEM (both $92\%$ vs.\ $84\%$). Averaged over fault types, \methodname{} has the highest Top-1 accuracy, ahead of DoWhy-GCM ($59\%$) and of CIRCA and RCD ($58\%$ each), and all its Avg@5 improvements are significant.
\begin{table*}[htbp]
\vspace{-5pt}
\centering
\caption{RCA method performance on RE1OB (Online Boutique with unimodal data).}
\vspace{-5pt}
\label{tab:rca_re1ob}
\resizebox{\textwidth}{!}{\setlength{\tabcolsep}{2pt}
\begin{tabular}{lccccccccccccccccccc}
\toprule
Method & \multicolumn{3}{c}{CPU} & \multicolumn{3}{c}{DELAY} & \multicolumn{3}{c}{DISK} & \multicolumn{3}{c}{LOSS} & \multicolumn{3}{c}{MEM} & \multicolumn{3}{c}{AVERAGE} & \multicolumn{1}{c}{STAT} \\
\cline{2-4} \cline{5-7} \cline{8-10} \cline{11-13} \cline{14-16} \cline{17-19} \cline{20-20}
 & Top-1 & Top-3 & Avg@5 & Top-1 & Top-3 & Avg@5 & Top-1 & Top-3 & Avg@5 & Top-1 & Top-3 & Avg@5 & Top-1 & Top-3 & Avg@5 & Top-1 & Top-3 & Avg@5 & p-value \\
\midrule
Smooth-Traversal & 0.04 & 0.25 & 0.19 & 0.08 & 0.20 & 0.16 & 0.04 & 0.08 & 0.07 & 0.04 & 0.17 & 0.13 & 0.08 & 0.16 & 0.13 & 0.06 & 0.17 & 0.14 & $1.2\times10^{-19}${\scriptsize (+)} \\
Simple-Traversal & 0.08 & 0.12 & 0.11 & 0.08 & 0.24 & 0.22 & 0.04 & 0.20 & 0.16 & 0.08 & 0.17 & 0.16 & 0.08 & 0.16 & 0.13 & 0.07 & 0.18 & 0.16 & $1.6\times10^{-19}${\scriptsize (+)} \\
Counterfactual & 0.04 & 0.12 & 0.13 & 0.08 & 0.28 & 0.27 & 0.20 & 0.24 & 0.30 & 0.12 & 0.25 & 0.27 & 0.08 & 0.16 & 0.20 & 0.11 & 0.21 & 0.23 & $9.5\times10^{-19}${\scriptsize (+)} \\
Cholesky & 0.08 & 0.20 & 0.22 & 0.08 & 0.24 & 0.22 & 0.12 & 0.28 & 0.28 & 0.16 & 0.32 & 0.34 & 0.12 & 0.16 & 0.24 & 0.11 & 0.24 & 0.26 & $6.2\times10^{-18}${\scriptsize (+)} \\
PC-PageRank & 0.00 & 0.12 & 0.13 & 0.08 & 0.16 & 0.19 & 0.16 & 0.44 & 0.42 & 0.04 & 0.36 & 0.31 & 0.08 & 0.24 & 0.26 & 0.07 & 0.26 & 0.26 & $1.2\times10^{-18}${\scriptsize (+)} \\
Score-Ordering & 0.08 & 0.42 & 0.34 & 0.12 & 0.32 & 0.33 & 0.12 & 0.32 & 0.36 & 0.08 & 0.29 & 0.37 & 0.08 & 0.24 & 0.27 & 0.10 & 0.32 & 0.33 & $1.1\times10^{-17}${\scriptsize (+)} \\
BARO & 0.12 & 0.75 & 0.60 & 0.24 & \underline{0.84} & 0.73 & 0.36 & \underline{0.96} & \underline{0.81} & 0.04 & 0.48 & 0.43 & 0.32 & \underline{0.96} & 0.82 & 0.22 & 0.80 & 0.68 & $2.0\times10^{-12}${\scriptsize (+)} \\
RCD & \textbf{0.88} & \underline{0.96} & \textbf{0.95} & 0.28 & 0.40 & 0.38 & 0.52 & 0.68 & 0.64 & \underline{0.32} & 0.48 & 0.50 & \textbf{0.92} & 0.92 & 0.92 & 0.58 & 0.69 & 0.68 & $2.6\times10^{-5}${\scriptsize (+)} \\
DoWhy-GCM & 0.64 & 0.73 & 0.72 & 0.64 & 0.72 & 0.74 & \underline{0.56} & 0.80 & 0.77 & \textbf{0.36} & 0.44 & 0.46 & 0.77 & 0.82 & 0.80 & \underline{0.59} & 0.70 & 0.70 & $6.3\times10^{-7}${\scriptsize (+)} \\
BRCD & 0.64 & \textbf{1.00} & 0.91 & 0.24 & \textbf{1.00} & 0.78 & 0.20 & 0.92 & 0.78 & 0.12 & \underline{0.52} & 0.50 & 0.68 & \textbf{1.00} & 0.93 & 0.38 & \underline{0.89} & 0.78 & $5.4\times10^{-5}${\scriptsize (+)} \\
CIRCA & \underline{0.76} & \underline{0.96} & \underline{0.92} & \underline{0.72} & \textbf{1.00} & \underline{0.93} & 0.20 & 0.92 & 0.74 & \underline{0.32} & \underline{0.52} & \underline{0.53} & \textbf{0.92} & \textbf{1.00} & \textbf{0.98} & 0.58 & 0.88 & \underline{0.82} & $0.004${\scriptsize (+)} \\
\midrule
\textbf{\methodname} & 0.52 & 0.92 & 0.85 & \textbf{0.76} & \textbf{1.00} & \textbf{0.95} & \textbf{0.64} & \textbf{1.00} & \textbf{0.92} & 0.28 & \textbf{0.68} & \textbf{0.66} & \underline{0.84} & \textbf{1.00} & \underline{0.96} & \textbf{0.61} & \textbf{0.92} & \textbf{0.87} & \textbf{N/A} \\
\bottomrule
\end{tabular}
}\vspace{-5pt}
\end{table*}

\subsubsection{Sock Shop (Unimodal)}

On RE1SS (Sock Shop, Table~\ref{tab:rca_re1ss}), \methodname{} demonstrates strong performance with $75\%$ Top-1 and $97\%$ Top-3 accuracy. The method achieves $84\%$ Top-1 on CPU faults and $80\%$ Top-1 on DELAY, DISK, and MEM faults. All improvements over baselines are statistically significant ($p < 0.05$).

\begin{table*}[htbp]
\vspace{-5pt}
\centering
\caption{RCA method performance on RE1SS (Sock Shop with unimodal data).}
\vspace{-5pt}
\label{tab:rca_re1ss}
\resizebox{\textwidth}{!}{\setlength{\tabcolsep}{2pt}
\begin{tabular}{lccccccccccccccccccc}
\toprule
Method & \multicolumn{3}{c}{CPU} & \multicolumn{3}{c}{DELAY} & \multicolumn{3}{c}{DISK} & \multicolumn{3}{c}{LOSS} & \multicolumn{3}{c}{MEM} & \multicolumn{3}{c}{AVERAGE} & \multicolumn{1}{c}{STAT} \\
\cline{2-4} \cline{5-7} \cline{8-10} \cline{11-13} \cline{14-16} \cline{17-19} \cline{20-20}
 & Top-1 & Top-3 & Avg@5 & Top-1 & Top-3 & Avg@5 & Top-1 & Top-3 & Avg@5 & Top-1 & Top-3 & Avg@5 & Top-1 & Top-3 & Avg@5 & Top-1 & Top-3 & Avg@5 & p-value \\
\midrule
Counterfactual & 0.16 & 0.24 & 0.22 & 0.08 & 0.12 & 0.12 & 0.16 & 0.24 & 0.21 & 0.04 & 0.04 & 0.04 & 0.16 & 0.16 & 0.18 & 0.12 & 0.16 & 0.15 & $1.3\times10^{-20}${\scriptsize (+)} \\
Smooth-Traversal & 0.12 & 0.24 & 0.22 & 0.00 & 0.24 & 0.22 & 0.04 & 0.16 & 0.14 & 0.00 & 0.16 & 0.11 & 0.04 & 0.20 & 0.17 & 0.04 & 0.20 & 0.17 & $1.1\times10^{-20}${\scriptsize (+)} \\
PC-PageRank & 0.20 & 0.24 & 0.27 & 0.20 & 0.28 & 0.25 & 0.20 & 0.40 & 0.32 & 0.24 & 0.44 & 0.39 & 0.04 & 0.24 & 0.18 & 0.18 & 0.32 & 0.28 & $2.0\times10^{-18}${\scriptsize (+)} \\
Cholesky & 0.08 & 0.20 & 0.23 & 0.08 & 0.20 & 0.21 & 0.04 & 0.32 & 0.34 & 0.12 & 0.36 & 0.40 & 0.12 & 0.40 & 0.37 & 0.09 & 0.30 & 0.31 & $2.0\times10^{-18}${\scriptsize (+)} \\
Simple-Traversal & 0.16 & 0.32 & 0.32 & 0.08 & 0.32 & 0.30 & 0.08 & 0.36 & 0.35 & 0.12 & 0.32 & 0.31 & 0.12 & 0.28 & 0.27 & 0.11 & 0.32 & 0.31 & $2.3\times10^{-18}${\scriptsize (+)} \\
Score-Ordering & 0.16 & 0.40 & 0.36 & 0.20 & 0.48 & 0.42 & 0.04 & 0.36 & 0.38 & 0.12 & 0.32 & 0.35 & 0.16 & 0.32 & 0.34 & 0.14 & 0.38 & 0.37 & $9.1\times10^{-18}${\scriptsize (+)} \\
RCD & 0.52 & 0.72 & 0.70 & 0.16 & 0.40 & 0.37 & 0.48 & 0.72 & 0.66 & 0.12 & 0.24 & 0.26 & 0.20 & 0.36 & 0.36 & 0.30 & 0.49 & 0.47 & $2.0\times10^{-14}${\scriptsize (+)} \\
DoWhy-GCM & \textbf{0.84} & 0.88 & 0.89 & \underline{0.73} & 0.89 & \underline{0.86} & \underline{0.69} & 0.77 & 0.76 & \textbf{0.56} & 0.72 & 0.69 & \underline{0.76} & 0.88 & \underline{0.86} & \underline{0.71} & 0.83 & 0.81 & $1.9\times10^{-4}${\scriptsize (+)} \\
BARO & 0.40 & \textbf{1.00} & 0.86 & 0.60 & \underline{0.92} & 0.84 & 0.40 & \underline{0.96} & 0.84 & 0.28 & 0.76 & 0.70 & 0.56 & \underline{0.92} & \underline{0.86} & 0.45 & 0.91 & 0.82 & $9.5\times10^{-6}${\scriptsize (+)} \\
BRCD & \underline{0.60} & \underline{0.96} & 0.90 & 0.56 & 0.88 & \underline{0.86} & 0.44 & 0.92 & 0.80 & 0.48 & \underline{0.88} & 0.78 & 0.48 & \underline{0.92} & 0.82 & 0.51 & 0.91 & 0.83 & $6.9\times10^{-6}${\scriptsize (+)} \\
CIRCA & \textbf{0.84} & \underline{0.96} & \underline{0.95} & 0.56 & \underline{0.92} & 0.85 & 0.36 & \underline{0.96} & \underline{0.85} & \textbf{0.56} & \underline{0.88} & \underline{0.81} & 0.60 & 0.88 & 0.82 & 0.58 & \underline{0.92} & \underline{0.86} & $1.9\times10^{-4}${\scriptsize (+)} \\
\midrule
\textbf{\methodname} & \textbf{0.84} & \textbf{1.00} & \textbf{0.96} & \textbf{0.80} & \textbf{0.96} & \textbf{0.93} & \textbf{0.80} & \textbf{1.00} & \textbf{0.95} & \underline{0.52} & \textbf{0.92} & \textbf{0.83} & \textbf{0.80} & \textbf{0.96} & \textbf{0.93} & \textbf{0.75} & \textbf{0.97} & \textbf{0.92} & \textbf{N/A} \\
\bottomrule
\end{tabular}
}\end{table*}
 
\subsubsection{Train Ticket (Unimodal)}

On RE1TT (Train Ticket, Table~\ref{tab:rca_re1tt}), the most complex system with 64 components, \methodname{} achieves $33\%$ Top-1 and $71\%$ Top-3 accuracy. CIRCA matches this Top-1 ($33\%$) with a lower Top-3 ($67\%$), and its Avg@5 difference from \methodname{} is not significant. The other baselines stay below $20\%$ Top-1. \methodname{} achieves $68\%$ Top-1 on MEM faults and $36\%$ on CPU. The increased system complexity poses challenges for all methods. DoWhy-GCM, for instance, exceeds its time limit on 51 of the 125 cases.

\begin{table*}[htbp]
\centering
\caption{RCA method performance on RE1TT (Train Ticket with unimodal data).}
\vspace{-5pt}
\label{tab:rca_re1tt}
\resizebox{\textwidth}{!}{\setlength{\tabcolsep}{2pt}
\begin{tabular}{lccccccccccccccccccc}
\toprule
Method & \multicolumn{3}{c}{CPU} & \multicolumn{3}{c}{DELAY} & \multicolumn{3}{c}{DISK} & \multicolumn{3}{c}{LOSS} & \multicolumn{3}{c}{MEM} & \multicolumn{3}{c}{AVERAGE} & \multicolumn{1}{c}{STAT} \\
\cline{2-4} \cline{5-7} \cline{8-10} \cline{11-13} \cline{14-16} \cline{17-19} \cline{20-20}
 & Top-1 & Top-3 & Avg@5 & Top-1 & Top-3 & Avg@5 & Top-1 & Top-3 & Avg@5 & Top-1 & Top-3 & Avg@5 & Top-1 & Top-3 & Avg@5 & Top-1 & Top-3 & Avg@5 & p-value \\
\midrule
PC-PageRank & 0.00 & 0.00 & 0.00 & 0.00 & 0.00 & 0.00 & 0.00 & 0.00 & 0.00 & 0.00 & 0.00 & 0.00 & 0.00 & 0.00 & 0.02 & 0.00 & 0.00 & 0.00 & $1.9\times10^{-19}${\scriptsize (+)} \\
Counterfactual & 0.00 & 0.08 & 0.08 & 0.00 & 0.04 & 0.02 & 0.00 & 0.04 & 0.03 & 0.00 & 0.00 & 0.00 & 0.00 & 0.04 & 0.02 & 0.00 & 0.04 & 0.03 & $6.0\times10^{-19}${\scriptsize (+)} \\
Simple-Traversal & 0.00 & 0.04 & 0.06 & 0.00 & 0.04 & 0.02 & 0.00 & 0.00 & 0.01 & 0.00 & 0.00 & 0.03 & 0.00 & 0.04 & 0.02 & 0.00 & 0.02 & 0.03 & $6.1\times10^{-19}${\scriptsize (+)} \\
Smooth-Traversal & 0.00 & 0.00 & 0.01 & 0.04 & 0.08 & 0.06 & 0.00 & 0.04 & 0.03 & 0.00 & 0.04 & 0.03 & 0.00 & 0.00 & 0.02 & 0.01 & 0.03 & 0.03 & $1.0\times10^{-18}${\scriptsize (+)} \\
Cholesky & 0.04 & 0.12 & 0.12 & 0.00 & 0.00 & 0.00 & 0.00 & 0.00 & 0.05 & 0.00 & 0.04 & 0.05 & 0.00 & 0.00 & 0.02 & 0.01 & 0.03 & 0.05 & $6.3\times10^{-19}${\scriptsize (+)} \\
Score-Ordering & 0.00 & 0.04 & 0.04 & 0.04 & 0.04 & 0.07 & 0.04 & 0.04 & 0.09 & 0.00 & 0.08 & 0.06 & 0.04 & 0.08 & 0.08 & 0.02 & 0.06 & 0.07 & $1.1\times10^{-18}${\scriptsize (+)} \\
RCD & 0.16 & 0.16 & 0.19 & 0.00 & 0.08 & 0.06 & \textbf{0.32} & 0.44 & 0.41 & 0.12 & 0.12 & 0.13 & 0.00 & 0.00 & 0.00 & 0.12 & 0.16 & 0.16 & $2.7\times10^{-14}${\scriptsize (+)} \\
DoWhy-GCM & 0.13 & 0.18 & 0.17 & 0.13 & 0.19 & 0.18 & 0.08 & 0.12 & 0.12 & 0.13 & 0.14 & 0.14 & 0.33 & 0.34 & 0.34 & 0.16 & 0.20 & 0.19 & $7.6\times10^{-16}${\scriptsize (+)} \\
BARO & 0.12 & 0.52 & 0.42 & 0.04 & 0.24 & 0.24 & 0.00 & 0.08 & 0.06 & 0.00 & 0.08 & 0.11 & 0.28 & 0.68 & 0.58 & 0.09 & 0.32 & 0.28 & $6.5\times10^{-16}${\scriptsize (+)} \\
BRCD & 0.12 & 0.32 & 0.35 & 0.12 & \underline{0.32} & 0.28 & 0.08 & 0.48 & 0.49 & 0.12 & 0.36 & 0.35 & 0.52 & \textbf{0.92} & \underline{0.83} & \underline{0.19} & 0.48 & 0.46 & $6.2\times10^{-6}${\scriptsize (+)} \\
CIRCA & \underline{0.32} & \underline{0.60} & \underline{0.58} & \textbf{0.40} & \textbf{0.68} & \textbf{0.68} & \underline{0.12} & \underline{0.52} & \underline{0.51} & \underline{0.16} & \textbf{0.68} & \textbf{0.56} & \underline{0.64} & \underline{0.88} & \underline{0.83} & \textbf{0.33} & \underline{0.67} & \underline{0.63} & $0.242${\scriptsize ($\approx$)} \\
\midrule
\textbf{\methodname} & \textbf{0.36} & \textbf{0.80} & \textbf{0.73} & \underline{0.24} & \textbf{0.68} & \underline{0.61} & \underline{0.12} & \textbf{0.64} & \textbf{0.59} & \textbf{0.24} & \underline{0.52} & \underline{0.54} & \textbf{0.68} & \textbf{0.92} & \textbf{0.89} & \textbf{0.33} & \textbf{0.71} & \textbf{0.67} & \textbf{N/A} \\
\bottomrule
\end{tabular}
}\vspace{-5pt}
\end{table*}

\subsubsection{Online Boutique (Multimodal)}

Table~\ref{tab:rca_re2ob} reports the results on RE2OB (Online Boutique with multimodal data), which Section~\ref{subsec:effectiveness} discusses under Performance by Fault Type.

\begin{table*}[htbp]
\centering
\caption{RCA method performance on RE2OB (Online Boutique with multimodal data).} \label{tab:rca_re2ob}
\resizebox{\textwidth}{!}{\setlength{\tabcolsep}{2pt}
\begin{tabular}{lcccccccccccccccccccccc}
\toprule
Method & \multicolumn{3}{c}{CPU} & \multicolumn{3}{c}{DELAY} & \multicolumn{3}{c}{DISK} & \multicolumn{3}{c}{LOSS} & \multicolumn{3}{c}{MEM} & \multicolumn{3}{c}{SOCKET} & \multicolumn{3}{c}{AVERAGE} & \multicolumn{1}{c}{STAT} \\
\cline{2-4} \cline{5-7} \cline{8-10} \cline{11-13} \cline{14-16} \cline{17-19} \cline{20-22} \cline{23-23}
 & Top-1 & Top-3 & Avg@5 & Top-1 & Top-3 & Avg@5 & Top-1 & Top-3 & Avg@5 & Top-1 & Top-3 & Avg@5 & Top-1 & Top-3 & Avg@5 & Top-1 & Top-3 & Avg@5 & Top-1 & Top-3 & Avg@5 & p-value \\
\midrule
Simple-Traversal & 0.07 & 0.13 & 0.12 & 0.00 & 0.13 & 0.12 & 0.00 & 0.07 & 0.05 & 0.07 & 0.13 & 0.16 & 0.13 & 0.13 & 0.17 & 0.00 & 0.13 & 0.11 & 0.04 & 0.12 & 0.12 & $1.2\times10^{-16}${\scriptsize (+)} \\
Smooth-Traversal & 0.07 & 0.13 & 0.11 & 0.00 & 0.07 & 0.07 & 0.00 & 0.13 & 0.12 & 0.07 & 0.20 & 0.19 & 0.07 & 0.33 & 0.24 & 0.00 & 0.20 & 0.13 & 0.03 & 0.18 & 0.14 & $1.4\times10^{-16}${\scriptsize (+)} \\
PC-PageRank & 0.00 & 0.20 & 0.23 & 0.13 & 0.20 & 0.24 & 0.00 & 0.07 & 0.09 & 0.00 & 0.13 & 0.13 & 0.00 & 0.07 & 0.08 & 0.00 & 0.20 & 0.12 & 0.02 & 0.14 & 0.15 & $3.3\times10^{-16}${\scriptsize (+)} \\
Counterfactual & 0.00 & 0.20 & 0.15 & 0.20 & 0.27 & 0.24 & 0.13 & 0.27 & 0.24 & 0.07 & 0.13 & 0.13 & 0.13 & 0.13 & 0.13 & 0.13 & 0.13 & 0.13 & 0.11 & 0.19 & 0.17 & $1.4\times10^{-15}${\scriptsize (+)} \\
Cholesky & 0.07 & 0.13 & 0.16 & 0.07 & 0.13 & 0.17 & 0.13 & 0.27 & 0.23 & 0.07 & 0.07 & 0.13 & 0.07 & 0.13 & 0.16 & 0.07 & 0.27 & 0.25 & 0.08 & 0.17 & 0.18 & $1.9\times10^{-15}${\scriptsize (+)} \\
RCLAgent & 0.20 & 0.20 & 0.20 & 0.07 & 0.33 & 0.28 & 0.13 & 0.27 & 0.23 & 0.27 & 0.47 & 0.41 & 0.27 & 0.40 & 0.35 & 0.13 & 0.27 & 0.27 & 0.18 & 0.32 & 0.29 & $4.9\times10^{-14}${\scriptsize (+)} \\
Score-Ordering & 0.07 & 0.20 & 0.21 & 0.00 & 0.13 & 0.12 & 0.07 & 0.27 & 0.36 & 0.20 & 0.33 & 0.36 & 0.07 & 0.20 & 0.25 & 0.00 & 0.33 & 0.41 & 0.07 & 0.24 & 0.29 & $6.7\times10^{-15}${\scriptsize (+)} \\
RCD & \textbf{0.93} & \underline{0.93} & \underline{0.93} & 0.47 & 0.60 & 0.63 & \underline{0.87} & \underline{0.93} & 0.92 & 0.33 & 0.47 & 0.44 & 0.47 & 0.53 & 0.53 & \underline{0.80} & \underline{0.93} & 0.89 & 0.64 & 0.73 & 0.72 & $2.4\times10^{-6}${\scriptsize (+)} \\
BRCD & 0.33 & 0.80 & 0.72 & 0.47 & \underline{0.80} & 0.73 & 0.27 & 0.87 & 0.73 & 0.20 & \underline{0.87} & 0.68 & 0.40 & 0.73 & 0.68 & 0.47 & 0.87 & 0.81 & 0.36 & 0.82 & 0.73 & $2.8\times10^{-10}${\scriptsize (+)} \\
BARO & 0.07 & 0.87 & 0.71 & 0.13 & \underline{0.80} & 0.68 & 0.80 & \textbf{1.00} & \underline{0.96} & 0.47 & \underline{0.87} & \underline{0.80} & 0.53 & \textbf{1.00} & 0.88 & 0.07 & 0.87 & 0.68 & 0.34 & 0.90 & 0.78 & $3.2\times10^{-11}${\scriptsize (+)} \\
CIRCA & 0.53 & \textbf{1.00} & 0.88 & \textbf{0.80} & \textbf{0.87} & \textbf{0.88} & 0.60 & \textbf{1.00} & 0.88 & 0.40 & 0.80 & 0.76 & 0.47 & \textbf{1.00} & 0.87 & 0.73 & \textbf{1.00} & \underline{0.95} & 0.59 & \underline{0.94} & 0.87 & $1.2\times10^{-6}${\scriptsize (+)} \\
DoWhy-GCM & \underline{0.87} & 0.87 & 0.87 & \underline{0.74} & 0.75 & \underline{0.78} & \textbf{1.00} & \textbf{1.00} & \textbf{1.00} & \underline{0.73} & 0.73 & 0.79 & \underline{0.94} & \underline{0.95} & \underline{0.95} & \textbf{0.93} & \underline{0.93} & 0.93 & \underline{0.87} & 0.87 & \underline{0.89} & $0.003${\scriptsize (+)} \\
\midrule
\textbf{\methodname} & \underline{0.87} & \textbf{1.00} & \textbf{0.97} & \textbf{0.80} & \underline{0.80} & \textbf{0.88} & \textbf{1.00} & \textbf{1.00} & \textbf{1.00} & \textbf{0.80} & \textbf{0.93} & \textbf{0.92} & \textbf{1.00} & \textbf{1.00} & \textbf{1.00} & \textbf{0.93} & \textbf{1.00} & \textbf{0.99} & \textbf{0.90} & \textbf{0.96} & \textbf{0.96} & \textbf{N/A} \\
\bottomrule
\end{tabular}
 }\end{table*}

\subsubsection{Sock Shop (Multimodal)}

On RE2SS (Sock Shop with multimodal data, Table~\ref{tab:rca_re2ss}), \methodname{} achieves $86\%$ Top-1 and $96\%$ Top-3 accuracy. It reaches $100\%$ Top-1 on MEM faults and $87\%$, $93\%$, and $80\%$ on CPU, DISK, and LOSS. RE2SS provides more metrics per component than RE1SS (Table~\ref{tab:datasets}), which may give a clearer internal--external separation.

\begin{table*}[htbp]
\vspace{-5pt}
\centering
\caption{RCA method performance on RE2SS (Sock Shop with multimodal data).}
\vspace{-5pt}
\label{tab:rca_re2ss}
\resizebox{\textwidth}{!}{\setlength{\tabcolsep}{2pt}
\begin{tabular}{lcccccccccccccccccccccc}
\toprule
Method & \multicolumn{3}{c}{CPU} & \multicolumn{3}{c}{DELAY} & \multicolumn{3}{c}{DISK} & \multicolumn{3}{c}{LOSS} & \multicolumn{3}{c}{MEM} & \multicolumn{3}{c}{SOCKET} & \multicolumn{3}{c}{AVERAGE} & \multicolumn{1}{c}{STAT} \\
\cline{2-4} \cline{5-7} \cline{8-10} \cline{11-13} \cline{14-16} \cline{17-19} \cline{20-22} \cline{23-23}
 & Top-1 & Top-3 & Avg@5 & Top-1 & Top-3 & Avg@5 & Top-1 & Top-3 & Avg@5 & Top-1 & Top-3 & Avg@5 & Top-1 & Top-3 & Avg@5 & Top-1 & Top-3 & Avg@5 & Top-1 & Top-3 & Avg@5 & p-value \\
\midrule
Counterfactual & 0.13 & 0.13 & 0.13 & 0.00 & 0.07 & 0.04 & 0.13 & 0.27 & 0.24 & 0.27 & 0.27 & 0.29 & 0.07 & 0.07 & 0.07 & 0.07 & 0.07 & 0.12 & 0.11 & 0.14 & 0.15 & $2.0\times10^{-15}${\scriptsize (+)} \\
Smooth-Traversal & 0.20 & 0.27 & 0.24 & 0.00 & 0.13 & 0.08 & 0.00 & 0.27 & 0.21 & 0.00 & 0.20 & 0.12 & 0.13 & 0.27 & 0.21 & 0.13 & 0.27 & 0.28 & 0.08 & 0.23 & 0.19 & $4.1\times10^{-15}${\scriptsize (+)} \\
PC-PageRank & 0.20 & 0.33 & 0.29 & 0.07 & 0.33 & 0.27 & 0.13 & 0.20 & 0.20 & 0.13 & 0.27 & 0.21 & 0.13 & 0.27 & 0.24 & 0.33 & 0.40 & 0.41 & 0.17 & 0.30 & 0.27 & $5.0\times10^{-14}${\scriptsize (+)} \\
Cholesky & 0.07 & 0.27 & 0.27 & 0.07 & 0.13 & 0.20 & 0.20 & 0.53 & 0.40 & 0.07 & 0.40 & 0.35 & 0.00 & 0.07 & 0.16 & 0.07 & 0.40 & 0.33 & 0.08 & 0.30 & 0.28 & $2.3\times10^{-14}${\scriptsize (+)} \\
Simple-Traversal & 0.27 & 0.40 & 0.37 & 0.07 & 0.47 & 0.40 & 0.07 & 0.60 & 0.44 & 0.00 & 0.40 & 0.25 & 0.20 & 0.40 & 0.35 & 0.20 & 0.33 & 0.31 & 0.13 & 0.43 & 0.35 & $5.5\times10^{-13}${\scriptsize (+)} \\
Score-Ordering & 0.20 & 0.40 & 0.41 & 0.07 & 0.40 & 0.37 & 0.13 & 0.47 & 0.43 & 0.13 & 0.53 & 0.45 & 0.13 & 0.40 & 0.36 & 0.13 & 0.53 & 0.49 & 0.13 & 0.46 & 0.42 & $3.8\times10^{-12}${\scriptsize (+)} \\
RCD & 0.60 & 0.87 & 0.81 & 0.33 & 0.40 & 0.47 & \underline{0.73} & \underline{0.87} & 0.81 & 0.13 & 0.47 & 0.41 & 0.33 & 0.60 & 0.59 & 0.47 & 0.80 & 0.75 & 0.43 & 0.67 & 0.64 & $2.0\times10^{-7}${\scriptsize (+)} \\
BARO & 0.00 & 0.73 & 0.60 & 0.00 & 0.67 & 0.57 & 0.67 & \textbf{0.93} & \underline{0.85} & 0.07 & \underline{0.93} & 0.71 & 0.13 & \textbf{1.00} & 0.79 & 0.00 & 0.60 & 0.51 & 0.14 & 0.81 & 0.67 & $1.9\times10^{-13}${\scriptsize (+)} \\
BRCD & 0.40 & \underline{0.93} & 0.81 & 0.53 & 0.73 & 0.71 & 0.27 & 0.67 & 0.64 & 0.33 & 0.87 & 0.69 & 0.40 & 0.80 & 0.73 & 0.33 & \textbf{0.87} & 0.73 & 0.38 & 0.81 & 0.72 & $1.9\times10^{-7}${\scriptsize (+)} \\
DoWhy-GCM & 0.74 & 0.89 & 0.86 & 0.62 & 0.65 & 0.67 & 0.70 & 0.77 & 0.76 & \underline{0.56} & 0.61 & 0.60 & \underline{0.81} & \underline{0.89} & 0.86 & \underline{0.69} & \underline{0.81} & \underline{0.79} & \underline{0.69} & 0.77 & 0.76 & $4.6\times10^{-6}${\scriptsize (+)} \\
CIRCA & \underline{0.80} & \textbf{1.00} & \underline{0.96} & \underline{0.67} & \underline{0.87} & \underline{0.85} & 0.53 & \textbf{0.93} & 0.84 & 0.47 & \underline{0.93} & \underline{0.84} & 0.60 & \textbf{1.00} & \underline{0.91} & 0.33 & \textbf{0.87} & 0.76 & 0.57 & \underline{0.93} & \underline{0.86} & $1.1\times10^{-6}${\scriptsize (+)} \\
\midrule
\textbf{\methodname} & \textbf{0.87} & \textbf{1.00} & \textbf{0.97} & \textbf{0.73} & \textbf{0.93} & \textbf{0.89} & \textbf{0.93} & \textbf{0.93} & \textbf{0.93} & \textbf{0.80} & \textbf{1.00} & \textbf{0.96} & \textbf{1.00} & \textbf{1.00} & \textbf{1.00} & \textbf{0.80} & \textbf{0.87} & \textbf{0.91} & \textbf{0.86} & \textbf{0.96} & \textbf{0.94} & \textbf{N/A} \\
\bottomrule
\end{tabular}
}\vspace{-5pt}
\end{table*}

\subsubsection{Train Ticket (Multimodal)}

On RE2TT (Train Ticket with multimodal data, Table~\ref{tab:rca_re2tt}), \methodname{} achieves $51\%$ Top-1 and $89\%$ Top-3 accuracy. Eight of the twelve baselines achieve at most $3\%$ Top-1 on this system, while \methodname{} achieves $73\%$ Top-1 on DISK faults and $67\%$ on MEM. BARO ($32\%$) and CIRCA ($31\%$) come closest in Top-1, and CIRCA has the highest baseline Top-3 ($72\%$). DoWhy-GCM exceeds its time limit on 59 of the 90 cases, and BRCD on 13.

\begin{table*}[htbp]
\vspace{-5pt}
\centering
\caption{RCA method performance on RE2TT (Train Ticket with multimodal data).}
\vspace{-5pt}
\label{tab:rca_re2tt}
\resizebox{\textwidth}{!}{\setlength{\tabcolsep}{2pt}
\begin{tabular}{lcccccccccccccccccccccc}
\toprule
Method & \multicolumn{3}{c}{CPU} & \multicolumn{3}{c}{DELAY} & \multicolumn{3}{c}{DISK} & \multicolumn{3}{c}{LOSS} & \multicolumn{3}{c}{MEM} & \multicolumn{3}{c}{SOCKET} & \multicolumn{3}{c}{AVERAGE} & \multicolumn{1}{c}{STAT} \\
\cline{2-4} \cline{5-7} \cline{8-10} \cline{11-13} \cline{14-16} \cline{17-19} \cline{20-22} \cline{23-23}
 & Top-1 & Top-3 & Avg@5 & Top-1 & Top-3 & Avg@5 & Top-1 & Top-3 & Avg@5 & Top-1 & Top-3 & Avg@5 & Top-1 & Top-3 & Avg@5 & Top-1 & Top-3 & Avg@5 & Top-1 & Top-3 & Avg@5 & p-value \\
\midrule
PC-PageRank & 0.00 & 0.00 & 0.00 & 0.00 & 0.00 & 0.00 & 0.00 & 0.00 & 0.00 & 0.00 & 0.00 & 0.00 & 0.00 & 0.00 & 0.00 & 0.00 & 0.00 & 0.00 & 0.00 & 0.00 & 0.00 & $7.7\times10^{-15}${\scriptsize (+)} \\
RCLAgent & 0.00 & 0.00 & 0.00 & 0.00 & 0.00 & 0.00 & 0.00 & 0.00 & 0.00 & 0.00 & 0.00 & 0.00 & 0.00 & 0.00 & 0.00 & 0.00 & 0.00 & 0.00 & 0.00 & 0.00 & 0.00 & $6.0\times10^{-15}${\scriptsize (+)} \\
Counterfactual & 0.00 & 0.00 & 0.00 & 0.00 & 0.00 & 0.00 & 0.00 & 0.07 & 0.07 & 0.00 & 0.00 & 0.01 & 0.00 & 0.07 & 0.05 & 0.00 & 0.00 & 0.00 & 0.00 & 0.02 & 0.02 & $7.7\times10^{-15}${\scriptsize (+)} \\
Simple-Traversal & 0.00 & 0.00 & 0.03 & 0.00 & 0.00 & 0.01 & 0.07 & 0.07 & 0.07 & 0.00 & 0.00 & 0.03 & 0.00 & 0.00 & 0.00 & 0.00 & 0.00 & 0.00 & 0.01 & 0.01 & 0.02 & $7.7\times10^{-15}${\scriptsize (+)} \\
Cholesky & 0.00 & 0.13 & 0.08 & 0.00 & 0.00 & 0.00 & 0.00 & 0.00 & 0.00 & 0.07 & 0.07 & 0.07 & 0.00 & 0.00 & 0.01 & 0.00 & 0.07 & 0.05 & 0.01 & 0.04 & 0.04 & $3.2\times10^{-14}${\scriptsize (+)} \\
Smooth-Traversal & 0.07 & 0.07 & 0.09 & 0.00 & 0.00 & 0.00 & 0.00 & 0.07 & 0.05 & 0.00 & 0.07 & 0.04 & 0.07 & 0.13 & 0.11 & 0.00 & 0.00 & 0.00 & 0.02 & 0.06 & 0.05 & $1.1\times10^{-14}${\scriptsize (+)} \\
Score-Ordering & 0.07 & 0.07 & 0.13 & 0.00 & 0.00 & 0.03 & 0.07 & 0.13 & 0.12 & 0.00 & 0.07 & 0.09 & 0.07 & 0.07 & 0.07 & 0.00 & 0.00 & 0.00 & 0.03 & 0.06 & 0.07 & $7.2\times10^{-14}${\scriptsize (+)} \\
RCD & \underline{0.27} & 0.40 & 0.40 & 0.07 & 0.13 & 0.11 & 0.13 & 0.20 & 0.19 & 0.00 & 0.07 & 0.04 & 0.07 & 0.27 & 0.21 & 0.07 & 0.27 & 0.20 & 0.10 & 0.22 & 0.19 & $1.1\times10^{-12}${\scriptsize (+)} \\
BRCD & 0.00 & 0.08 & 0.12 & 0.00 & 0.20 & 0.20 & 0.07 & 0.07 & 0.14 & 0.00 & 0.27 & 0.29 & 0.00 & 0.27 & 0.26 & 0.00 & 0.21 & 0.17 & 0.01 & 0.18 & 0.20 & $5.8\times10^{-14}${\scriptsize (+)} \\
DoWhy-GCM & \underline{0.27} & 0.29 & 0.29 & 0.28 & \underline{0.30} & 0.30 & 0.28 & 0.30 & 0.30 & 0.14 & 0.16 & 0.16 & 0.34 & 0.36 & 0.36 & 0.08 & 0.10 & 0.10 & 0.23 & 0.25 & 0.25 & $2.8\times10^{-12}${\scriptsize (+)} \\
BARO & 0.20 & \underline{0.73} & \underline{0.65} & 0.00 & 0.20 & 0.19 & \textbf{0.73} & \textbf{1.00} & \textbf{0.95} & 0.00 & 0.20 & 0.19 & \underline{0.60} & \textbf{1.00} & \textbf{0.92} & \underline{0.40} & 0.47 & 0.52 & \underline{0.32} & 0.60 & 0.57 & $1.5\times10^{-5}${\scriptsize (+)} \\
CIRCA & \underline{0.27} & 0.60 & 0.55 & \underline{0.40} & \textbf{0.87} & \underline{0.75} & \underline{0.60} & \underline{0.93} & \underline{0.87} & \textbf{0.27} & \underline{0.53} & \underline{0.45} & 0.27 & 0.80 & 0.69 & 0.07 & \underline{0.60} & \underline{0.57} & 0.31 & \underline{0.72} & \underline{0.65} & $3.5\times10^{-6}${\scriptsize (+)} \\
\midrule
\textbf{\methodname} & \textbf{0.47} & \textbf{0.87} & \textbf{0.80} & \textbf{0.47} & \textbf{0.87} & \textbf{0.79} & \textbf{0.73} & \textbf{1.00} & \textbf{0.95} & \underline{0.20} & \textbf{0.67} & \textbf{0.52} & \textbf{0.67} & \underline{0.93} & \underline{0.88} & \textbf{0.53} & \textbf{1.00} & \textbf{0.88} & \textbf{0.51} & \textbf{0.89} & \textbf{0.80} & \textbf{N/A} \\
\bottomrule
\end{tabular}
}\vspace{-5pt}
\end{table*}

\subsubsection{Online Boutique (Code-Level Faults)}

On RE3OB (Online Boutique with code-level faults, Table~\ref{tab:rca_re3ob}), \methodname{} achieves $90\%$ Top-1 and $98\%$ Top-3 accuracy. Performance is highest on F2, F3, and F4 faults ($100\%$ Top-1), with strong performance on F5 ($83\%$ Top-1). F1 achieves $67\%$ Top-1. CIRCA matches the average accuracy of \methodname{} ($90\%$ Top-1 and $98\%$ Top-3), and their Avg@5 difference is not significant.

\begin{table*}[htbp]
\vspace{-5pt}
\centering
\caption{RCA method performance on RE3OB (Online Boutique with code-level faults).}
\vspace{-5pt}
\label{tab:rca_re3ob}
\resizebox{\textwidth}{!}{\setlength{\tabcolsep}{2pt}
\begin{tabular}{lccccccccccccccccccc}
\toprule
Method & \multicolumn{3}{c}{F1} & \multicolumn{3}{c}{F2} & \multicolumn{3}{c}{F3} & \multicolumn{3}{c}{F4} & \multicolumn{3}{c}{F5} & \multicolumn{3}{c}{AVERAGE} & \multicolumn{1}{c}{STAT} \\
\cline{2-4} \cline{5-7} \cline{8-10} \cline{11-13} \cline{14-16} \cline{17-19} \cline{20-20}
 & Top-1 & Top-3 & Avg@5 & Top-1 & Top-3 & Avg@5 & Top-1 & Top-3 & Avg@5 & Top-1 & Top-3 & Avg@5 & Top-1 & Top-3 & Avg@5 & Top-1 & Top-3 & Avg@5 & p-value \\
\midrule
Smooth-Traversal & 0.11 & 0.11 & 0.11 & 0.00 & 0.33 & 0.27 & 0.00 & 0.17 & 0.13 & 0.17 & 0.17 & 0.17 & 0.00 & 0.17 & 0.13 & 0.06 & 0.19 & 0.16 & $1.3\times10^{-5}${\scriptsize (+)} \\
Counterfactual & 0.11 & 0.11 & 0.11 & 0.33 & 0.33 & 0.47 & 0.00 & 0.17 & 0.13 & 0.00 & 0.17 & 0.13 & 0.00 & 0.00 & 0.00 & 0.09 & 0.16 & 0.17 & $1.3\times10^{-5}${\scriptsize (+)} \\
Cholesky & 0.11 & 0.33 & 0.27 & 0.00 & 0.33 & 0.20 & 0.33 & 0.50 & 0.43 & 0.00 & 0.00 & 0.10 & 0.33 & 0.33 & 0.37 & 0.16 & 0.30 & 0.27 & $6.3\times10^{-5}${\scriptsize (+)} \\
RCLAgent & 0.25 & 0.25 & 0.25 & 0.00 & 0.00 & 0.00 & 0.17 & 0.33 & 0.30 & 0.33 & 0.50 & 0.47 & 0.17 & 0.50 & 0.40 & 0.18 & 0.32 & 0.28 & $8.5\times10^{-5}${\scriptsize (+)} \\
PC-PageRank & 0.11 & 0.22 & 0.33 & 0.00 & 0.00 & 0.27 & 0.00 & 0.17 & 0.13 & 0.00 & 0.33 & 0.27 & 0.33 & 0.50 & 0.57 & 0.09 & 0.24 & 0.31 & $6.1\times10^{-5}${\scriptsize (+)} \\
Simple-Traversal & 0.11 & 0.22 & 0.20 & 0.33 & 0.67 & 0.60 & 0.17 & 0.50 & 0.43 & 0.17 & 0.33 & 0.30 & 0.17 & 0.33 & 0.30 & 0.19 & 0.41 & 0.37 & $6.3\times10^{-5}${\scriptsize (+)} \\
Score-Ordering & 0.22 & 0.44 & 0.44 & 0.33 & 0.67 & 0.67 & 0.17 & 0.67 & 0.50 & 0.17 & 0.17 & 0.30 & 0.17 & 0.50 & 0.43 & 0.21 & 0.49 & 0.47 & $1.4\times10^{-4}${\scriptsize (+)} \\
BARO & 0.00 & \textbf{1.00} & 0.80 & 0.00 & \textbf{1.00} & 0.73 & 0.17 & \textbf{1.00} & 0.83 & 0.00 & \textbf{1.00} & 0.77 & 0.00 & 0.50 & 0.40 & 0.03 & \underline{0.90} & 0.71 & $0.001${\scriptsize (+)} \\
RCD & \underline{0.33} & 0.33 & 0.33 & 0.67 & 0.67 & 0.67 & \textbf{1.00} & \textbf{1.00} & \textbf{1.00} & 0.67 & 0.67 & 0.67 & \textbf{0.83} & \underline{0.83} & 0.87 & 0.70 & 0.70 & 0.71 & $0.011${\scriptsize (+)} \\
DoWhy-GCM & \textbf{0.67} & 0.67 & 0.67 & \underline{0.69} & \underline{0.75} & 0.75 & \underline{0.84} & \underline{0.85} & \underline{0.85} & \underline{0.68} & \underline{0.71} & 0.78 & \textbf{0.83} & \underline{0.83} & 0.83 & \underline{0.74} & 0.76 & 0.78 & $0.023${\scriptsize (+)} \\
BRCD & \underline{0.33} & 0.78 & 0.67 & 0.67 & \textbf{1.00} & \underline{0.93} & 0.67 & 0.67 & 0.67 & 0.67 & \textbf{1.00} & \underline{0.87} & \underline{0.67} & \textbf{1.00} & \underline{0.90} & 0.60 & 0.89 & 0.81 & $0.011${\scriptsize (+)} \\
CIRCA & \textbf{0.67} & \underline{0.89} & \underline{0.82} & \textbf{1.00} & \textbf{1.00} & \textbf{1.00} & \textbf{1.00} & \textbf{1.00} & \textbf{1.00} & \textbf{1.00} & \textbf{1.00} & \textbf{1.00} & \textbf{0.83} & \textbf{1.00} & \textbf{0.97} & \textbf{0.90} & \textbf{0.98} & \underline{0.96} & $0.414${\scriptsize ($\approx$)} \\
\midrule
\textbf{\methodname} & \textbf{0.67} & \underline{0.89} & \textbf{0.87} & \textbf{1.00} & \textbf{1.00} & \textbf{1.00} & \textbf{1.00} & \textbf{1.00} & \textbf{1.00} & \textbf{1.00} & \textbf{1.00} & \textbf{1.00} & \textbf{0.83} & \textbf{1.00} & \textbf{0.97} & \textbf{0.90} & \textbf{0.98} & \textbf{0.97} & \textbf{N/A} \\
\bottomrule
\end{tabular}
}\vspace{-5pt}
\end{table*}

\subsubsection{Sock Shop (Code-Level Faults)}

On RE3SS (Sock Shop with code-level faults, Table~\ref{tab:rca_re3ss}), \methodname{} achieves $20\%$ Top-1 and $79\%$ Top-3 accuracy. The high Top-3 relative to Top-1 indicates the root cause consistently appears among top candidates, with the highest performance on F1 ($60\%$ Top-1). This is the most difficult dataset for \methodname{}. Its Top-1 accuracy is below that of CIRCA ($37\%$) and DoWhy-GCM ($25\%$) and, after rounding, equals that of Smooth-Traversal ($20\%$), although its Top-3 accuracy is the highest of all methods. It ranks no F2 or F4 root cause first, and its differences from six baselines, including CIRCA, are not statistically significant.

\begin{table*}[htbp]
\centering
\caption{RCA method performance on RE3SS (Sock Shop with code-level faults).}
\label{tab:rca_re3ss}
\resizebox{\textwidth}{!}{\setlength{\tabcolsep}{2pt}
\begin{tabular}{lcccccccccccccccc}
\toprule
Method & \multicolumn{3}{c}{F1} & \multicolumn{3}{c}{F2} & \multicolumn{3}{c}{F3} & \multicolumn{3}{c}{F4} & \multicolumn{3}{c}{AVERAGE} & \multicolumn{1}{c}{STAT} \\
\cline{2-4} \cline{5-7} \cline{8-10} \cline{11-13} \cline{14-16} \cline{17-17}
 & Top-1 & Top-3 & Avg@5 & Top-1 & Top-3 & Avg@5 & Top-1 & Top-3 & Avg@5 & Top-1 & Top-3 & Avg@5 & Top-1 & Top-3 & Avg@5 & p-value \\
\midrule
Counterfactual & 0.10 & 0.30 & 0.30 & \underline{0.00} & 0.00 & 0.07 & \underline{0.10} & 0.10 & 0.14 & 0.14 & 0.43 & 0.34 & 0.09 & 0.21 & 0.21 & $0.001${\scriptsize (+)} \\
BRCD & 0.00 & 0.10 & 0.14 & \underline{0.00} & 0.33 & 0.40 & 0.00 & 0.20 & 0.24 & 0.00 & 0.29 & 0.20 & 0.00 & 0.23 & 0.25 & $5.2\times10^{-4}${\scriptsize (+)} \\
RCD & 0.00 & 0.30 & 0.28 & \underline{0.00} & 0.00 & 0.33 & 0.00 & 0.30 & 0.28 & 0.14 & 0.43 & 0.46 & 0.04 & 0.26 & 0.34 & $0.016${\scriptsize (+)} \\
DoWhy-GCM & \underline{0.61} & \underline{0.73} & 0.71 & \underline{0.00} & 0.00 & 0.00 & \textbf{0.20} & 0.40 & 0.38 & 0.18 & 0.39 & 0.36 & \underline{0.25} & 0.38 & 0.36 & $0.094${\scriptsize ($\approx$)} \\
PC-PageRank & 0.20 & 0.40 & 0.42 & \underline{0.00} & 0.33 & 0.27 & 0.00 & 0.20 & 0.18 & 0.29 & 0.57 & 0.57 & 0.12 & 0.38 & 0.36 & $0.013${\scriptsize (+)} \\
Cholesky & 0.20 & 0.50 & 0.44 & \underline{0.00} & 0.33 & 0.27 & \textbf{0.20} & 0.40 & 0.44 & 0.14 & 0.43 & 0.43 & 0.14 & 0.42 & 0.39 & $0.023${\scriptsize (+)} \\
Score-Ordering & 0.10 & 0.40 & 0.36 & \underline{0.00} & 0.00 & 0.07 & \underline{0.10} & \underline{0.50} & 0.42 & \textbf{0.57} & \textbf{0.86} & \textbf{0.77} & 0.19 & 0.44 & 0.40 & $0.094${\scriptsize ($\approx$)} \\
Simple-Traversal & 0.10 & 0.50 & 0.36 & \underline{0.00} & 0.33 & 0.27 & \textbf{0.20} & 0.40 & 0.40 & \underline{0.43} & \underline{0.71} & \underline{0.69} & 0.18 & 0.49 & 0.43 & $0.094${\scriptsize ($\approx$)} \\
BARO & 0.20 & 0.70 & 0.62 & \underline{0.00} & 0.33 & 0.40 & \textbf{0.20} & 0.40 & 0.44 & 0.29 & 0.57 & 0.60 & 0.17 & 0.50 & 0.52 & $0.094${\scriptsize ($\approx$)} \\
Smooth-Traversal & 0.10 & 0.70 & 0.52 & \textbf{0.33} & \textbf{1.00} & \textbf{0.80} & \underline{0.10} & \underline{0.50} & 0.48 & 0.29 & 0.43 & 0.40 & 0.20 & 0.66 & 0.55 & $0.094${\scriptsize ($\approx$)} \\
CIRCA & \textbf{0.90} & \textbf{1.00} & \textbf{0.98} & \textbf{0.33} & \underline{0.67} & \underline{0.73} & \underline{0.10} & \underline{0.50} & \underline{0.52} & 0.14 & \underline{0.71} & 0.57 & \textbf{0.37} & \underline{0.72} & \textbf{0.70} & $0.887${\scriptsize ($\approx$)} \\
\midrule
\textbf{\methodname} & 0.60 & \textbf{1.00} & \underline{0.88} & \underline{0.00} & \textbf{1.00} & \underline{0.73} & \textbf{0.20} & \textbf{0.60} & \textbf{0.60} & 0.00 & 0.57 & 0.51 & 0.20 & \textbf{0.79} & \underline{0.68} & \textbf{N/A} \\
\bottomrule
\end{tabular}
}\end{table*}

\subsubsection{Train Ticket (Code-Level Faults)}

On RE3TT (Train Ticket with code-level faults, Table~\ref{tab:rca_re3tt}), \methodname{} achieves $90\%$ Top-1 and $100\%$ Top-3 accuracy. Perfect accuracy is observed on F1, F2, and F4 faults ($100\%$ Top-1), and $60\%$ on F3. The second-best method, CIRCA, reaches $58\%$ Top-1 and $97\%$ Top-3, and BARO only $20\%$ Top-1 despite $100\%$ Top-3. The other baselines reach at most $11\%$ Top-1 (DoWhy-GCM exceeds its time limit on 21 of the 30 cases). With only 3 to 10 cases per fault type, this result suggests, but does not establish, that \methodname{} remains accurate on larger systems.

\begin{table*}[htbp]
\centering
\caption{RCA method performance on RE3TT (Train Ticket with code-level faults).}
\label{tab:rca_re3tt}
\resizebox{\textwidth}{!}{\setlength{\tabcolsep}{2pt}
\begin{tabular}{lcccccccccccccccc}
\toprule
Method & \multicolumn{3}{c}{F1} & \multicolumn{3}{c}{F2} & \multicolumn{3}{c}{F3} & \multicolumn{3}{c}{F4} & \multicolumn{3}{c}{AVERAGE} & \multicolumn{1}{c}{STAT} \\
\cline{2-4} \cline{5-7} \cline{8-10} \cline{11-13} \cline{14-16} \cline{17-17}
 & Top-1 & Top-3 & Avg@5 & Top-1 & Top-3 & Avg@5 & Top-1 & Top-3 & Avg@5 & Top-1 & Top-3 & Avg@5 & Top-1 & Top-3 & Avg@5 & p-value \\
\midrule
BRCD & 0.00 & 0.00 & 0.00 & 0.00 & 0.00 & 0.00 & 0.00 & 0.00 & 0.00 & 0.00 & 0.00 & 0.00 & 0.00 & 0.00 & 0.00 & $2.4\times10^{-6}${\scriptsize (+)} \\
PC-PageRank & 0.00 & 0.00 & 0.00 & 0.00 & 0.00 & 0.00 & 0.00 & 0.00 & 0.00 & 0.00 & 0.00 & 0.00 & 0.00 & 0.00 & 0.00 & $2.4\times10^{-6}${\scriptsize (+)} \\
RCLAgent & 0.00 & 0.00 & 0.00 & 0.00 & 0.00 & 0.00 & 0.00 & 0.00 & 0.00 & 0.00 & 0.00 & 0.00 & 0.00 & 0.00 & 0.00 & $2.4\times10^{-6}${\scriptsize (+)} \\
RCD & 0.00 & 0.00 & 0.03 & 0.00 & 0.00 & 0.00 & 0.00 & 0.00 & 0.00 & 0.00 & 0.00 & 0.00 & 0.00 & 0.00 & 0.01 & $2.4\times10^{-6}${\scriptsize (+)} \\
Smooth-Traversal & 0.00 & 0.00 & 0.00 & 0.00 & 0.00 & 0.00 & 0.00 & 0.00 & 0.04 & 0.00 & 0.00 & 0.00 & 0.00 & 0.00 & 0.01 & $2.4\times10^{-6}${\scriptsize (+)} \\
Cholesky & 0.00 & \underline{0.14} & 0.11 & 0.00 & 0.00 & 0.00 & 0.00 & 0.00 & 0.00 & 0.00 & 0.00 & 0.00 & 0.00 & 0.04 & 0.03 & $2.4\times10^{-6}${\scriptsize (+)} \\
Counterfactual & 0.00 & 0.00 & 0.00 & 0.14 & 0.14 & 0.14 & 0.00 & 0.00 & 0.00 & 0.00 & 0.00 & 0.00 & 0.04 & 0.04 & 0.04 & $2.4\times10^{-6}${\scriptsize (+)} \\
Score-Ordering & 0.00 & 0.00 & 0.00 & 0.00 & 0.14 & 0.11 & 0.10 & 0.10 & 0.14 & 0.00 & 0.00 & 0.00 & 0.03 & 0.06 & 0.06 & $2.7\times10^{-6}${\scriptsize (+)} \\
Simple-Traversal & 0.00 & \underline{0.14} & 0.09 & 0.00 & 0.14 & 0.17 & 0.00 & 0.10 & 0.08 & 0.00 & 0.00 & 0.00 & 0.00 & 0.10 & 0.08 & $2.5\times10^{-6}${\scriptsize (+)} \\
DoWhy-GCM & 0.01 & 0.04 & 0.04 & 0.15 & \underline{0.17} & 0.17 & 0.11 & 0.12 & 0.12 & 0.18 & \underline{0.20} & 0.20 & 0.11 & 0.13 & 0.13 & $9.5\times10^{-6}${\scriptsize (+)} \\
BARO & 0.14 & \textbf{1.00} & 0.83 & 0.14 & \textbf{1.00} & 0.74 & \underline{0.20} & \textbf{1.00} & \underline{0.80} & 0.33 & \textbf{1.00} & 0.87 & 0.20 & \textbf{1.00} & 0.81 & $5.9\times10^{-4}${\scriptsize (+)} \\
CIRCA & \underline{0.57} & \textbf{1.00} & \underline{0.91} & \underline{0.71} & \textbf{1.00} & \underline{0.94} & \underline{0.20} & \underline{0.90} & \underline{0.80} & \underline{0.83} & \textbf{1.00} & \underline{0.97} & \underline{0.58} & \underline{0.97} & \underline{0.91} & $9.1\times10^{-4}${\scriptsize (+)} \\
\midrule
\textbf{\methodname} & \textbf{1.00} & \textbf{1.00} & \textbf{1.00} & \textbf{1.00} & \textbf{1.00} & \textbf{1.00} & \textbf{0.60} & \textbf{1.00} & \textbf{0.90} & \textbf{1.00} & \textbf{1.00} & \textbf{1.00} & \textbf{0.90} & \textbf{1.00} & \textbf{0.97} & \textbf{N/A} \\
\bottomrule
\end{tabular}
}\end{table*}
  \section{IT-Score Saturation and Marginal Score Ordering}\label{app:it-saturation}

\subsection{Case Study: IT-Score Saturation Problem}\label{subsec:it-saturation}

\citet{orchardroot} propose using Information-Theoretic (IT) anomaly scores for root cause analysis. Given $k$ pre-failure observations, the IT score is estimated as
\begin{equation}\label{eq:it-score-est}
    \hat{S}(x) = -\log \frac{|\{i : \tau(x_i) \geq \tau(x)\}|}{k},
\end{equation}
where $\tau$ is a feature function (e.g., z-score or IQR-based score). When the count is zero, Laplace smoothing replaces the ratio by $0.5/(k+0.5)$.

However, this formulation suffers from a \emph{saturation problem} in practice. When a failure occurs, multiple components exhibit anomalies that exceed their normal variation. If the maximum post-failure $\tau$-score exceeds \emph{all} pre-failure $\tau$-scores, the count in Eq.~\ref{eq:it-score-est} becomes zero, and the smoothed estimate takes the same constant value $\log(2k+1)$ for all such metrics.

\begin{figure}[!tbp]
\centering
\includegraphics[width=0.85\textwidth,height=0.62\textheight,keepaspectratio]{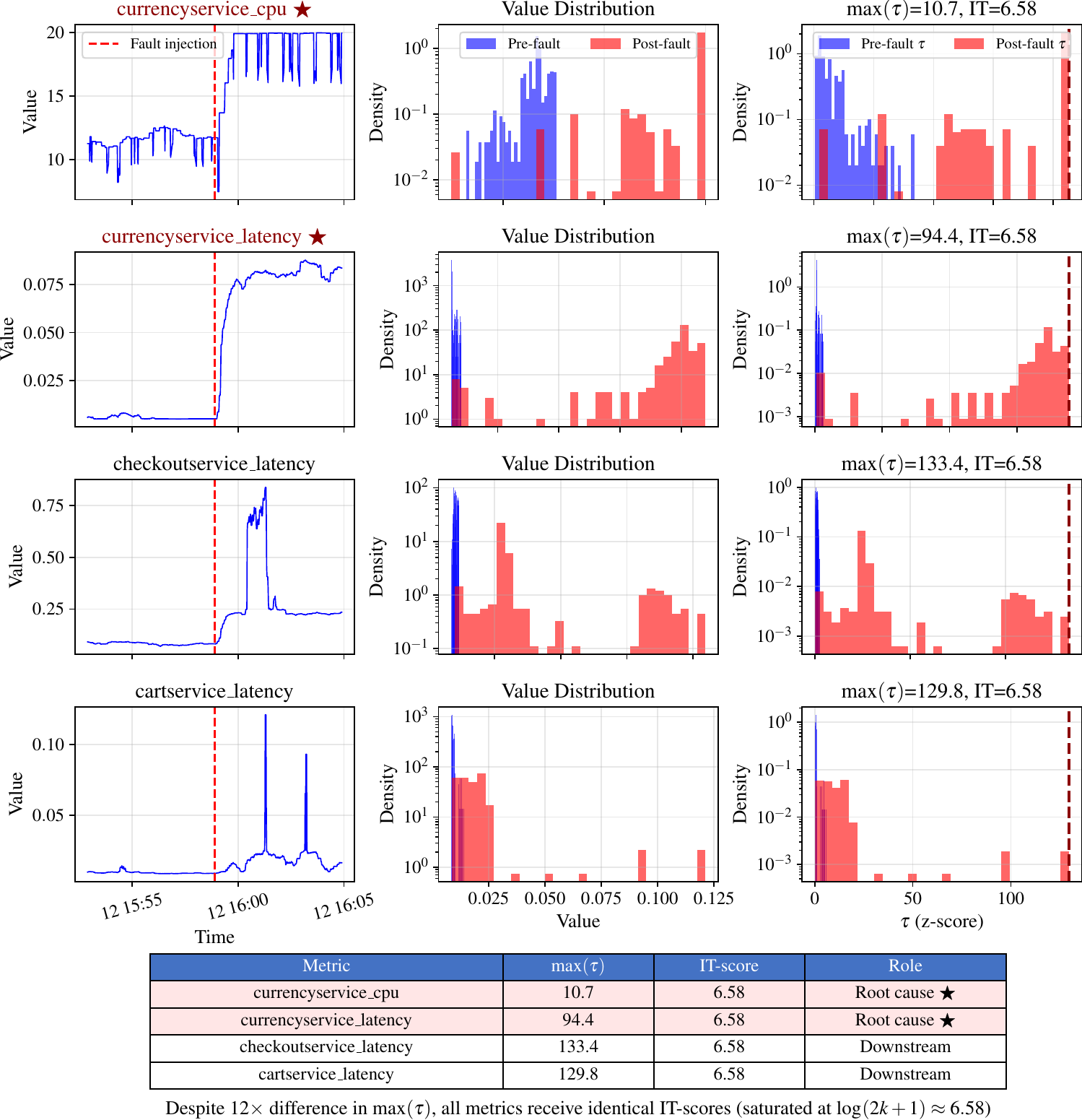}
\caption{IT-score saturation on a disk fault in \service{currencyservice} (RE1OB, a different case from the RE2OB failure case of Table~\ref{tab:case-study}). Only four of the case's metrics are shown. Despite vastly different anomaly magnitudes (max $\tau$ ranging from 10.7 to 133.4), all metrics receive identical IT scores of approximately 6.58. The root cause metrics (\service{currencyservice\_cpu} and \service{currencyservice\_latency}, marked $\star$) cannot be distinguished from downstream effects (\service{checkoutservice\_latency}, \service{cartservice\_latency}) based on IT scores alone. The summary table quantifies the saturation: a 12$\times$ difference in $\tau$ magnitude yields identical scores.}
\label{fig:it-saturation}
\end{figure}

Figure~\ref{fig:it-saturation} demonstrates this phenomenon on a failure case from the RE1OB dataset (Online Boutique). A disk fault was injected into \service{currencyservice}, causing anomalies in both root cause metrics (\service{currencyservice\_cpu}, \service{currencyservice\_latency}) and downstream latency metrics (\service{checkoutservice\_latency}, \service{cartservice\_latency}). Despite max $\tau$ values ranging from 10.7 to 133.4 (a 12$\times$ difference), all metrics receive identical IT scores of approximately $6.58$. This occurs because with $k=360$ pre-failure samples, an observation that does not exceed the reference maximum has a count of at least one and scores at most $\log 360 \approx 5.89$. Every observation that exceeds the reference maximum has a count of zero and receives the smoothed constant $\log(2k+1) = \log 721 \approx 6.58$, regardless of its magnitude (Theorem~\ref{thm:it-saturation}).

The fundamental limitation is that IT scores measure \emph{whether} an anomaly occurred (a binary notion of ``surprise''), not \emph{how severe} it is relative to other anomalies. This renders IT scores ineffective for ranking candidate root causes in systems where failures propagate to multiple downstream components.

\subsection{Theory of IT-Score Saturation and Marginal Score Ordering}\label{subsec:it-saturation-theory}

The Information-Theoretic (IT) score transforms deviation through tail probability:
\begin{equation}\label{eq:it-score}
    S_{\text{IT}}(x) = -\log P(\tau(X) \geq \tau(x))
\end{equation}
where $\tau: \mathcal{X} \to \mathbb{R}$ is a feature function (e.g., z-score) and $P$ is the reference distribution. 

Given $k$ reference observations $\{x_1, \ldots, x_k\}$ from the normal operating period, let $N(x) = |\{i \in [k] : \tau(x_i) \geq \tau(x)\}|$, where $[k] = \{1, \ldots, k\}$. The empirical IT score with Laplace smoothing of the zero count is
\begin{equation}\label{eq:empirical-it}
    \hat{S}_{\text{IT}}(x) =
    \begin{cases}
      -\log \dfrac{N(x)}{k} & \text{if } N(x) \geq 1,\\[6pt]
      -\log \dfrac{0.5}{k + 0.5} = \log(2k+1) & \text{if } N(x) = 0.
    \end{cases}
\end{equation}

\begin{theorem}[IT-Score Saturation]\label{thm:it-saturation}
Let $\hat{S}_{\text{IT}}$ be the empirical IT score of Eq.~\ref{eq:empirical-it} with $k$ reference observations.
\begin{enumerate}[nolistsep,leftmargin=*]
    \item \textbf{Nonzero count.} If $N(x) \geq 1$, then $0 \leq \hat{S}_{\text{IT}}(x) \leq \log k$.
    \item \textbf{Saturation.} If $\tau(x) > \max_{i \in [k]} \tau(x_i)$, then $N(x) = 0$ and $\hat{S}_{\text{IT}}(x) = \log(2k+1)$, a constant that does not depend on $\tau(x)$ and is the largest attainable score.
    \item \textbf{Non-injectivity.} All observations $x^{(1)}, \ldots, x^{(m)}$ that exceed the reference maximum receive the same score $\log(2k+1)$, whatever their $\tau$-values.
\end{enumerate}
The zero-count regime of Part~2 is the saturation regime. $\hat{S}_{\text{IT}}$ is a nondecreasing step function of $\tau(x)$ that is constant on the regime $\tau(x) > \max_{i \in [k]} \tau(x_i)$. Appendix~\ref{app:proofs} proves this, including smoothing of every count.
\end{theorem}

\begin{corollary}[IT Scores are Non-Injective Under Saturation]\label{cor:it-non-injective}
If $m \geq 2$ post-failure observations exceed the reference maximum, i.e., $\tau(x^{(j)}) > \max_{i \in [k]} \tau(x_i)$ for $j = 1, \ldots, m$, the empirical IT score assigns all of them the value $\log(2k+1)$. It is therefore not injective (Definition~\ref{def:injectivity}) whenever their deviations differ, however large the difference. It restates Part~3 of Theorem~\ref{thm:it-saturation}, proved in Appendix~\ref{app:proofs}.
\end{corollary}

\begin{theorem}[Counterexample to Marginal Score Ordering]\label{thm:counter}
There exist systems satisfying Assumptions~\ref{ax:direction}--\ref{ax:isolation} in which a fault propagates through external properties and an affected component has a strictly higher external anomaly score $S^E$ than the root cause, even though each component is scored against its own baseline. When a caller makes $k$ sequential calls to the faulty callee for each request and the per-call latency fluctuations are independent, the caller's baseline-normalized deviation is $\sqrt{k}$ times the callee's. See Appendix~\ref{app:proofs} for the proof.
\end{theorem}

\begin{corollary}[Deviation-Based Scores Preserve the Deviation Order]\label{cor:deviation-rca}
A deviation-based scorer orders observations exactly by their deviations $|x - c|$ from a common reference, with ties only between equal deviations. It therefore represents deviation magnitudes without the saturation collapse of IT scores. Identifying the root cause additionally requires the internal--external decomposition (Theorem~\ref{thm:cie-correct}), since the largest external deviation need not belong to the root cause (Theorem~\ref{thm:counter}). Injectivity is thus necessary (Theorem~\ref{thm:rca-necessity}) but not sufficient. See Appendix~\ref{app:proofs} for the proof.
\end{corollary}

\begin{corollary}[IT Scores Fail for RCA Under Saturation]\label{cor:it-fails}
Let components be ranked by max-pooled empirical IT scores with a common reference size $k$, and let exactly $m \geq 2$ components each have a property whose post-failure $\tau$-value exceeds its reference maximum. These $m$ components tie at the largest attainable score $\log(2k+1)$, and under any tie-breaking rule that is exchangeable among tied components (i.e., selects each of them first with equal probability), the root cause is ranked first with probability at most $1/m$. See Appendix~\ref{app:proofs} for the proof.
\end{corollary}

\FloatBarrier
 \section{Internal--External Property Labeling}\label{app:labeling}

\subsection{Theoretical Foundation}

This section provides theoretical grounding for the internal--external property classification. The distinction between internal and external properties aligns with established observability frameworks in distributed systems, as in~\citet{li2022causal} and the Google SRE golden signals~\cite{googlesre}. We classify properties as:
\begin{itemize}[leftmargin=*,nolistsep]
\item \textbf{External properties} (Quality-of-Service metrics): Observable outcomes at service boundaries that other components can detect through interactions, including response time (latency), request rate, and error rate.
\item \textbf{Internal properties} (Resource metrics): Local component states not directly observable by other components. Examples: CPU usage, memory utilization, disk I/O, socket count. The framework extends naturally to other internal data sources, such as log patterns (error counts, exception frequencies, log message anomalies) and trace spans internal to a component (processing durations, retry counts, queue depths), as long as these observables remain local to the component and are not directly visible through inter-component interfaces.
\end{itemize}

This classification reflects the RED method (Rate, Errors, Duration) for service-level metrics~\cite{wilkie2018red} and the USE method (Utilization, Saturation, Errors) for resource metrics~\cite{gregg2013thinking}. The key insight is that external properties represent \emph{what other components observe}, while internal properties represent \emph{why the component behaves that way}.

\subsection{Empirical Validation of the Assumptions}

We validate Assumptions~\ref{ax:direction}--\ref{ax:isolation} on the 735 failure cases in the RCAEval benchmark.

\begin{table}[ht]
\centering
\caption{Temporal precedence analysis by fault type. ``Int First'' is the percentage of cases where the internal onset precedes or equals the external onset (ties, 15 of the 698 cases, count as internal-first). Negative $\Delta t$ values indicate internal-first ordering. Temporal order is weak, non-causal evidence for Assumption~\ref{ax:direction}.}
\label{tab:temporal-precedence}
\begin{tabular}{lrrr}
\toprule
\textbf{Fault Type} & \textbf{$n$} & \textbf{Int First} & \textbf{Median $\Delta t$} \\
\midrule
\multicolumn{4}{l}{\textit{Resource Faults}} \\
CPU & 117 & 76.9\% & $-7.0$s \\
DISK & 113 & 73.5\% & $-6.0$s \\
MEM & 120 & 65.8\% & $-4.0$s \\
SOCKET & 45 & 64.4\% & $-5.0$s \\
\midrule
\multicolumn{4}{l}{\textit{Network Faults}} \\
DELAY & 113 & 27.4\% & $+11.0$s \\
LOSS & 109 & 41.3\% & $+9.0$s \\
\midrule
\multicolumn{4}{l}{\textit{Code-Level Faults}} \\
F1--F5 & 81 & 82.7\% & $-15.0$s \\ \midrule
\textbf{Overall} & \textbf{698} & \textbf{60.7\%} & $\mathbf{-4.0}$s \\
\bottomrule
\end{tabular}
\end{table}
\subsubsection{Assumption~\ref{ax:direction}: Intracomponent Directionality}

Assumption~\ref{ax:direction} states that internal properties are causal ancestors of external properties within each component. We validate this through temporal precedence analysis. For each failure case, we detect the first post-failure timestamp where the anomaly score (z-score) exceeds a threshold of 3.0 for both internal and external properties at the root cause component. We compute $\Delta t = t_{\text{internal}} - t_{\text{external}}$, where negative values indicate internal anomalies precede external anomalies.

\mypara{Results.} Across 698 cases with detectable anomalies in both property types, internal anomalies precede or equal external anomalies in \textbf{60.7\%} of cases. Table~\ref{tab:temporal-precedence} shows results by fault type.
Internal onsets strictly precede external ones in 409 of the 698 cases (58.6\%). Resource faults (CPU, DISK, MEM, SOCKET) partially support Assumption~\ref{ax:direction}. 64--77\% of them show internal-first ordering, with median $\Delta t$ of $-4$ to $-7$ seconds, but in 23--36\% of them (34\% for MEM) the external onset comes first, which the assumption does not predict. Code-level faults exhibit the strongest internal-first pattern (82.7\%), plausibly because application bugs affect processing before service degradation becomes observable. Network faults (DELAY, LOSS) show the opposite pattern, because their cause lies in the network rather than in the component's internal properties. They fall outside the fault model of Theorem~\ref{thm:cie-correct}, which treats the fault as an intervention on an internal property. Temporal precedence does not establish causal ancestry, so this analysis is weak, non-causal evidence for the assumption.

\subsubsection{Assumption~\ref{ax:isolation}: Intercomponent Isolation}
Assumption~\ref{ax:isolation} states that faults propagate between components only through external properties. We validate this by comparing internal scores of root cause versus affected components. For each failure case, we compute $S^I$ for the root cause and all other components. We measure the ratio $S^I(\text{root cause}) / \max_a S^I(C_a)$, where $C_a$ ranges over the non-root components with a nonzero external score ($S^E > 0$), which we call affected.

\mypara{Results.} Across all 735 cases, the root cause's internal score $S^I$ exceeds the maximum internal score among affected components in 66.5\% of cases, with a median ratio of 2.73$\times$. This separation supports the assumption that faults manifest internally at the root cause while propagating mainly through external properties to downstream components. The separation fails in about a third of the cases, so the separation event of Theorem~\ref{thm:cie-correct} often does not hold in RCAEval. Appendix~\ref{app:contention} analyzes non-root internal anomalies among resource faults.

Figure~\ref{fig:property-decomposition} visualizes this separation. Root cause components (red) lie to the right, at higher $S^I$, while most non-root components (gray) concentrate at low $S^I$ and $S^E$. A small group of non-root components reaches very large $S^E$ with low $S^I$, reflecting external anomalies without corresponding internal anomalies. Some non-root components also reach high $S^I$, the isolation violations discussed below.

\begin{figure}[!htbp]
\centering
\includegraphics[width=0.6\linewidth]{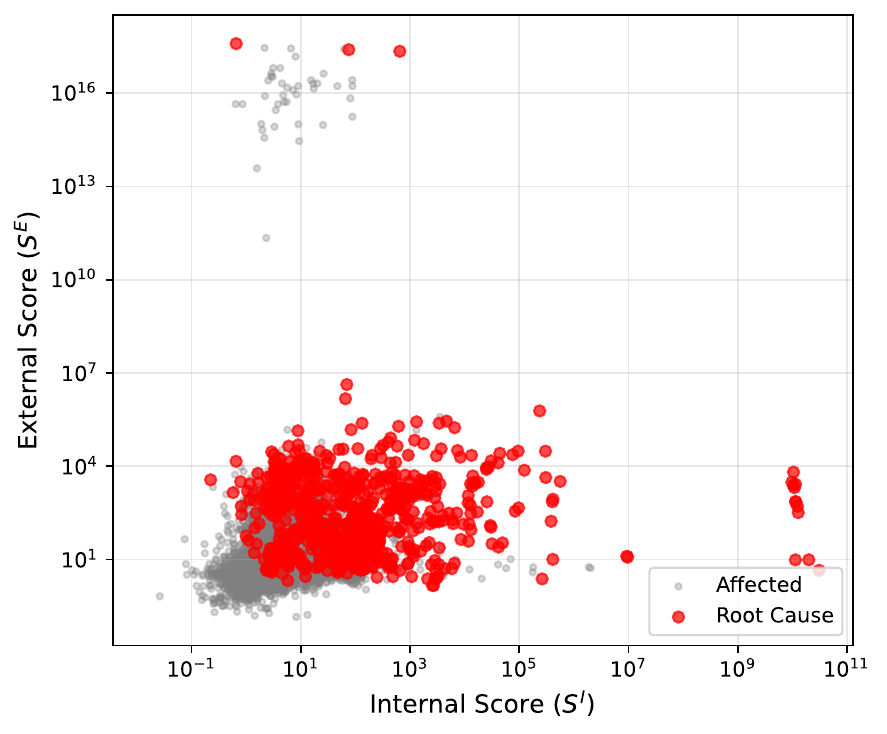}
\caption{Internal ($S^I$, x-axis) vs.\ external ($S^E$, y-axis) scores of all components across the 735 failure cases, on log scales. Red points are root causes, and gray points (legend ``Affected'') are non-root components with nonzero scores, the only ones a log scale can show. Root causes lie at higher $S^I$, while most non-root components concentrate at low $S^I$ and $S^E$.} \label{fig:property-decomposition}
\end{figure}

\subsubsection{Scope and Applicability of the Component-Property Model (CPM) Assumptions}\label{app:cpm-scope}

The assumptions of the CPM establish a foundation for principled root cause analysis. Rather than rigid requirements, these assumptions characterize system architectures where \methodname{} can be applied with high confidence. This subsection identifies real-world systems satisfying the CPM assumptions, discusses deployment scenarios where the assumptions may not hold, and gives practitioners practical guidance for assessing \methodname{}'s applicability to their specific systems.

\mypara{Systems Satisfying the CPM Assumptions.}
Many modern system architectures are designed to approximate Assumptions~\ref{ax:direction} and~\ref{ax:isolation}. In these systems, the assumptions are more likely to hold, although rarely exactly.

\textit{Containerized Microservices with Resource Quotas.} Modern container orchestration platforms enforce strict resource isolation through kernel-level mechanisms. Kubernetes~\cite{kubernetes2024resources} uses Linux cgroups to implement CPU and memory limits, which limit how much one pod's resource use can affect another pod's resource availability. cgroups do not isolate every shared resource (e.g., disk I/O bandwidth, caches, memory bandwidth), so containerized deployments satisfy Assumption~\ref{ax:isolation} only approximately. The RCAEval benchmark uses containerized Kubernetes deployments, and the analysis above shows that internal separation fails in about a third of its cases.

\textit{Actor Model Systems.} Systems based on the actor model exemplify the CPM assumptions through state encapsulation. Akka~\cite{akka2024actors} and Erlang OTP~\cite{erlang2024otp} enforce that each actor maintains private state, modifiable only by that actor. All inter-actor communication occurs through asynchronous message passing, representing external properties and interactions. This architectural constraint ensures that one actor cannot directly affect another's internal state, satisfying Assumption~\ref{ax:isolation}. The sequential message processing within each actor provides temporal causality (internal then external), satisfying Assumption~\ref{ax:direction}.

\textit{Service-Oriented Architecture.} SOA principles~\cite{erl1900service,soa2025reference} enforce component isolation through service interfaces. Services interact through explicit contracts, with no access to internal implementation details. This separation creates boundaries that satisfy Assumption~\ref{ax:isolation}. Each service's internal state (e.g., caches, db) remains invisible to other services, while external properties (e.g., response time) propagate through the network.

\textit{Serverless with Tenant Isolation.} Recent serverless platforms offer explicit tenant isolation modes. AWS Lambda's tenant isolation feature~\cite{aws2025lambda} creates separate execution environments for each tenant, preventing resource contention and interference. While this trades off cold start overhead for stronger isolation guarantees, it ensures that one tenant's function execution cannot directly affect another tenant's internal metrics, satisfying Assumption~\ref{ax:isolation}.

\mypara{Violations Without Co-location.} The assumptions can also fail without shared hardware. Retries raise the request rate and hence the CPU usage of a callee. A slow callee lets open connections, sockets, and threads pile up in its callers, and socket count is labeled internal. Request rate, which we label external, drives the resource usage of its own component, an external-to-internal edge that Assumption~\ref{ax:direction} excludes. In such cases, non-root components can show internal anomalies, the separation event of Lemma~\ref{lem:internal-separation} can fail, and the guarantees of Section~\ref{sec:ranking} do not apply. The labeling used in our experiments therefore violates Assumption~\ref{ax:direction} knowingly for workload metrics, and the theory describes the evaluated systems only approximately.

\textbf{While many systems approximate the CPM assumptions, some systems do not. The CPM assumptions characterize a design space where \methodname{} is most reliable, so practitioners can decide whether they can apply \methodname{} to their production systems with confidence.}

While the CPM assumptions provide a foundation for effective RCA, future work can extend \methodname{}'s applicability and robustness, including: (1) \textbf{Evaluation Under Resource Contention.} Appendix~\ref{app:contention} quantifies the degradation on the RCAEval cases that exhibit contention. A controlled evaluation on over-subscribed Kubernetes clusters, where the assumptions are violated by design, would establish performance bounds as interference grows.
(2) \textbf{Assumption Violation Detection.} Develop diagnostic tests to automatically detect when the CPM assumptions are violated in production systems, enabling runtime warnings and guiding infrastructure improvements. (3) \textbf{Adaptive Property Classification.} Develop methods to automatically classify metrics as internal or external based on system topology and metric semantics, reducing manual labeling burden for new systems. (4) Integrate assumption compliance monitoring with infrastructure management, enabling real-time enforcement of isolation guarantees and automated remediation when violations occur.

 \section{Performance Under Isolation Violations}\label{app:contention}

Assumption~\ref{ax:isolation} (Intercomponent Isolation) may be violated in shared infrastructure where one component's resource usage indirectly affects another's internal state (e.g., memory faults triggering OS swapping that causes disk I/O spikes across co-located services). Retries, queue buildup, and connection or thread exhaustion can have the same effect without co-location (Appendix~\ref{app:cpm-scope}). To quantify the degradation, we flag cases with a non-root internal anomaly among the 405 resource-fault cases of RCAEval (CPU, MEM, DISK, and SOCKET faults). A case is flagged if some non-root service has an internal metric whose post-failure mean exceeds its pre-failure mean by more than three pre-failure standard deviations and by a factor above $1.5$. The criterion does not check co-location, so it captures isolation violations of any origin, not only resource contention (Table~\ref{tab:contention}).

\begin{table}[ht]
\centering
\caption{\methodname{} on resource-fault cases with and without a flagged non-root internal anomaly (isolation violation). Accuracy is the fraction of cases in each group (per-case, not macro-averaged), and differences use unrounded accuracies. Top-1 degrades by roughly 20\,pp in flagged cases, but Top-3 remains above 86\%.}
\label{tab:contention}
\begin{tabular}{lccc}
\toprule
Group & $n$ & Top-1 & Top-3 \\
\midrule
Not flagged & 292 & 0.764 & 0.966 \\
Flagged     & 113 & 0.566 & 0.867 \\
\midrule
Degradation     &     & $-$19.7\,pp & $-$9.8\,pp \\ \bottomrule
\end{tabular}
\end{table}

Of the 405 resource faults, 113 (27.9\%) are flagged (Table~\ref{tab:contention}). Isolation violations are thus common in RCAEval. They degrade Top-1 by roughly 20 percentage points, yet Top-3 remains above 86\%, indicating that \methodname{} still ranks the root cause near the top even when the isolation assumption is partially violated. For comparison, CIRCA, the best baseline overall, reaches 0.53 Top-1 on the unflagged cases and 0.45 on the flagged ones (BARO 0.32 and 0.31), so the margin of \methodname{} over CIRCA shrinks from 23 to 12 percentage points. DoWhy-GCM, the strongest baseline on the unflagged cases (0.65), drops to 0.40. Because 76 of the 113 flagged cases come from Train Ticket, these comparisons do not isolate the effect of isolation violations. Practitioners can detect such violations by checking whether multiple components exhibit elevated internal scores.

\section{Single-Sample Diagnosis}\label{app:single-sample}

\methodname{} scores each property by the deviation of its peak post-failure value from the pre-failure reference (Section~\ref{subsec:anomaly_scoring}). With a single post-failure sample, the peak is that sample, so a single sample suffices to compute all scores. To evaluate this setting, we keep the full pre-failure data and randomly select one post-failure sample per metric in each of the 735 failure cases, repeating the draw with three random seeds. The sample is drawn from anywhere in the post-failure window and independently for each metric, so it is neither the first sample after the failure nor a snapshot taken at a common time. The experiment therefore shows how much information one post-failure value per metric carries, not how accurate \methodname{} is at the onset of the failure.

\begin{table}[ht]
\centering
\caption{Effect of the amount of post-failure data on \methodname{}. The single-sample row averages three random draws (standard deviation 0.04 for Top-1).}
\label{tab:single-sample}
\begin{tabular}{lcccc}
\toprule
Post-failure data & Top-1 & Top-3 & Top-5 & Avg@5 \\
\midrule
1 sample            & 0.65 & 0.84 & 0.90 & 0.81 \\
10\% (36--210 samples) & 0.81 & 0.97 & 0.99 & 0.93 \\
100\% (all samples) & 0.68 & 0.91 & 0.98 & 0.87 \\
\bottomrule
\end{tabular}
\end{table}

With a single sample, \methodname{} achieves 0.65 Top-1 (Table~\ref{tab:single-sample}) and 0.81 Avg@5, somewhat below its full-data accuracy (0.68 and 0.87). Its Top-1 still exceeds that of the best baseline with all post-failure data (CIRCA, 0.56), although its Avg@5 is slightly lower (0.81 vs.\ 0.82). Both of its values far exceed those of BARO (0.19 and 0.63). Accuracy peaks with a short window of the first 10\% of the post-failure data (36 to 210 one-second samples, depending on the dataset), consistent with Figure~\ref{fig:robustness-data-lengths}. A short window already captures the onset of the fault, while a longer one lets propagated anomalies accumulate at affected components. A practical deployment strategy is therefore to run \methodname{} shortly after failure detection, without waiting for a long observation window.
 \section{Reproducibility}\label{app:reproducibility}

This section describes how to reproduce the main tables and figures (Tables~\ref{tab:rca_rcaeval}--\ref{tab:ablation}, the per-dataset tables, and Figures~\ref{fig:efficiency-boxplot}--\ref{fig:robustness-data-lengths}) and lists the scripts for the single-sample and contention analyses.

\subsection{Environment Setup}

\begin{lstlisting}
# Requires Python 3.10+
pip install -e .
\end{lstlisting}

\noindent The code runs on standard hardware (8 CPUs, 16\,GB RAM). No GPU is required. The complete benchmark using pre-computed results runs in under 5 minutes. RCD, BRCD, CIRCA, and DoWhy-GCM ran in a separate harness on CPU cluster nodes with 4 cores and 16\,GB RAM per case (2 cores and 8\,GB for CIRCA, Appendix~\ref{app:baselines}).

\subsection{Reproducing Tables from Pre-computed Results}

The archive \texttt{precomputed\_outputs.tar.gz} contains pre-computed results for all 735 failure cases and all methods, including the rankings of RCD, BRCD, CIRCA, and DoWhy-GCM. Extracting it creates \texttt{output/} and \texttt{output-robustness/}. Records of cases over the time limit carry a \texttt{fallback} field, and the evaluation credits them with the expected accuracy of a random ranking. The following commands generate the tables.
\begin{lstlisting}
tar xzf precomputed_outputs.tar.gz
\end{lstlisting}

\mypara{Table~\ref{tab:rca_rcaeval} (Overall Performance).}
\begin{lstlisting}
python run_eval.py \
  --methods cholesky baro counterfactual pc_pagerank \
    rclagent score_ordering simple_traversal \
    smooth_traversal rcd brcd circa dowhy_gcm cie \
  --sort --ignore-top5 --average-only --ignore-std \
  --text-width --pvalue --latex table.tex
\end{lstlisting}

\noindent Expected output: \methodname{} achieves 68\%, 91\%, and 87\% (Top-1, Top-3, Avg@5), and the best baseline, CIRCA, achieves 56\%, 87\%, and 82\%.

\mypara{Table~\ref{tab:rca_re2ob} (RE2OB Detailed Results).}
\begin{lstlisting}
python run_eval.py --filter re2ob \
  --methods cholesky baro counterfactual pc_pagerank \
    rclagent score_ordering simple_traversal \
    smooth_traversal rcd brcd circa dowhy_gcm cie \
  --sort --ignore-top5 --ignore-std --text-width \
  --pvalue --latex table.tex
\end{lstlisting}

\noindent Expected output: \methodname{} achieves 90\% Top-1, 96\% Top-3, 96\% Avg@5 on RE2OB.

\mypara{Table~\ref{tab:robustness} (Robustness Analysis).}

Anomaly scorers:
\begin{lstlisting}
python run_eval.py --methods cie cie_iqr \
  cie_it cie_it_iqr --ablation --latex table.tex
\end{lstlisting}

\noindent Expected output: z-score 68\%, IT-z-score 14\%, and IT-IQR 13\% Top-1.

\noindent Pooling strategies:
\begin{lstlisting}
python run_eval.py --methods cie cie_mean \
  cie_sum --ablation --latex table.tex
\end{lstlisting}

\noindent Expected output: $\max$ 68\%, mean 68\%, and sum 69\% Top-1.

\noindent Root cause scoring:
\begin{lstlisting}
python run_eval.py --methods cie_conj cie \
  --ablation --ignore-std --latex table.tex
\end{lstlisting}

\noindent Expected output: conjunctive 66\% Top-1 and 80\% Top-3, additive 68\% Top-1 and 91\% Top-3.

\mypara{Table~\ref{tab:ablation} (Ablation Study).}
\begin{lstlisting}
python run_eval.py --methods cie_marginal cie \
  cie_internal cie_external \
  --ablation --latex table.tex
\end{lstlisting}

\noindent Expected output: $\text{\methodname}_\text{Marginal}$ 19\%, $\text{\methodname}_\text{Internal}$ 15\%, $\text{\methodname}_\text{External}$ 54\%, and \methodname{} 68\% Top-1.

\mypara{Per-Dataset Tables (Appendix~\ref{appendix:rca-effectiveness}).}
\begin{lstlisting}
# Replace DATASET with: re1ob, re1ss, re1tt,
# re2ob, re2ss, re2tt, re3ob, re3ss, re3tt
python run_eval.py --filter DATASET \
  --methods cholesky baro counterfactual pc_pagerank \
    rclagent score_ordering simple_traversal \
    smooth_traversal rcd brcd circa dowhy_gcm cie \
  --sort --ignore-top5 --ignore-std --text-width \
  --pvalue --latex table.tex
\end{lstlisting}

\mypara{Figure~\ref{fig:efficiency-boxplot} (Efficiency).}
\begin{lstlisting}
python run_eval.py --time \
  --efficiency-json efficiency.json
python visualize_box_plot_efficiency.py \
  --output efficiency_boxplot.pdf
\end{lstlisting}

\noindent Expected output: median running time of 8.2\,ms for \methodname{} and 6.3\,ms for BARO.

\mypara{Figure~\ref{fig:robustness-data-lengths} (Sensitivity).}
\begin{lstlisting}
python robustness_plot.py
\end{lstlisting}

\noindent Expected output: \methodname{} 81\% Top-1 at 10\% and 68\% at 100\% of the post-failure data, and BARO 46\% and 19\%.

\mypara{Single-Sample Diagnosis (Appendix~\ref{app:single-sample}).}
\begin{lstlisting}
python eval_single_sample.py
\end{lstlisting}

\mypara{Contention (Appendix~\ref{app:contention}).}
\begin{lstlisting}
python contention_degrade.py
\end{lstlisting}

\subsection{Re-running Experiments from Scratch}

To re-run \methodname{} and seven of the baselines from scratch:

\begin{lstlisting}
# Run CIE and seven baselines (approx. 2 hours)
python run_bench.py --methods cie baro pc_pagerank \
  cholesky counterfactual score_ordering \
  simple_traversal smooth_traversal --force

# Sensitivity and single-sample runs
for m in cie baro; do for r in 0.1 0.2 0.3 0.4 0.5 \
  0.6 0.7 0.8 0.9 1.0; do python run_robustness.py \
  --method $m --ratio $r; done; done
for s in 1 2 3; do python run_single_sample.py --seed $s; done

# RCD, BRCD, CIRCA, DoWhy-GCM
pip install -e ".[external]"
python run_external.py --methods rcd brcd circa dowhy_gcm

# Generate evaluation tables
python run_eval.py --latex results.tex
\end{lstlisting}
\noindent The adapters that ran RCD, BRCD, CIRCA, and DoWhy-GCM are in \texttt{external\_baselines/}. Running them took about 1,400 hours of cluster jobs with 4 cores (2 for CIRCA), 1,080 of these hours for DoWhy-GCM, including 158 cases stopped at the four-hour limit.

\noindent \textbf{Note:} RCLAgent requires an API key for the LLM provider set in its configuration file and takes several hours due to LLM inference.
  
\end{document}